\documentclass{elsarticle}

\makeatletter
\def\ps@pprintTitle{%
 \let\@oddhead\@empty
 \let\@evenhead\@empty
 \def\@oddfoot{}%
 \let\@evenfoot\@oddfoot}
\makeatother

\usepackage{color,soul} 

\usepackage{amsmath,amssymb,amsfonts,amsthm}
\usepackage{mathtools}
\usepackage{xspace}
\usepackage{tikz}
\usetikzlibrary{patterns}
\usepackage{subcaption}
\usepackage{color,soul} 
\usepackage{fleqn}
\usepackage{enumerate}
\usepackage{hyperref}
\usepackage{microtype}

\theoremstyle{plain}
    \newtheorem{theorem}{Theorem}
    \newtheorem{corollary}[theorem]{Corollary}
    \newtheorem{lemma}[theorem]{Lemma}
    \newtheorem{proposition}[theorem]{Proposition}
    \newtheorem{observation}[theorem]{Observation}
\newcounter{claim}[theorem]
\newtheorem{claim}[claim]{Claim}

    \newtheorem{fact}{Fact}
    \newtheorem{definition}[theorem]{Definition}
    \newtheorem{example}[theorem]{Example}
    \newtheorem*{theorem*}{Theorem}
    \newtheorem*{corollary*}{Corollary}
    \newtheorem*{lemma*}{Lemma}
    \newtheorem*{proposition*}{Proposition}
    \newtheorem*{claim*}{Claim}
    \newtheorem*{definition*}{Definition}
    \newtheorem*{example*}{Example}

\DeclarePairedDelimiter\abs{\lvert}{\rvert}

\mathchardef\mhyphen="2D
\newcommand*{\range}[2]{\{#1,\ldots,#2\}}

\newcommand*{\rationals}{\mathbb{Q}}
\newcommand*{\naturals}{\mathbb{N}}

\newcommand{\spii}{\ \ }

\newcommand{\spiv}{\ \ \ \ }
\newcommand{\spv}{\ \ \ \ \ }

\newcommand{\stexttt}[1]{{\tiny\tt #1}}

\newcommand{\cc}[1]{{\mbox{\textnormal{\textsf{#1}}}}\xspace}  

\newcommand{\NP}{\cc{NP}}
\newcommand{\FPT}{\cc{FPT}}
\newcommand{\XP}{\cc{XP}}

\def\inst{{\cal I}}
\def\cert{{\cal C}}
\def\A{{\bf A}}
\def\B{{\bf B}}
\def\C{{\bf C}}
\def\D{{\bf D}}
\def\I{{\bf I}}

\def\G{{\bf \Gamma}}
\def\T{{\bf \Theta}}

\def\CDClang{\ensuremath{\B_{\rm CDC}}}
\def\IAlang{\ensuremath{\B_{\rm IA}}}
\def\BAlang{\ensuremath{\B_{{\rm BA}_d}}}
\def\CDCdisjlang{\ensuremath{\B^{\vee =}_{\rm CDC}}}
\def\IAdisjlang{\ensuremath{\B^{\vee =}_{\rm IA}}}
\def\BAdisjlang{\ensuremath{\B^{\vee =}_{{\rm BA}_d}}}

\newcommand{\boole}[1]{\langle {#1} \rangle_{\rm b}}

\newcommand{\disj}[1]{{#1}^{\vee =}}
\newcommand{\age}[1]{\ensuremath{\textnormal{Age}(#1)}}
\newcommand{\aut}[1]{\ensuremath{\textnormal{Aut}(#1)}}

\newcommand{\Orb}[1]{\ensuremath{\textnormal{Orb}(#1)}}
\newcommand{\Orbi}{\ensuremath{\cal O}}

\begin{document}
\begin{frontmatter}

\title{Solving Infinite-Domain CSPs Using the Patchwork Property\footnote{Parts of this article appeared in the 
    proceedings of the 35th AAAI Conference on Artificial Intelligence (AAAI~2021)~\cite{Dabrowski:etal:aaai2021-pp}.}}

\author[a1]{Konrad K. Dabrowski\corref{cor1}}
\cortext[cor1]{Corresponding author.}
\ead{k.k.dabrowski@leeds.ac.uk}

\author[a2]{Peter Jonsson}
\ead{peter.jonsson@liu.se}

\author[a1]{Sebastian Ordyniak}
\ead{sordyniak@gmail.com}

\author[a2]{George~Osipov}
\ead{george.osipov@liu.se}

\address[a1]{School of Computing, University of Leeds, UK}
  
\address[a2]{Department of Computer and Information Science, Link\"opings
  Universitet, Sweden.}

\begin{abstract}
  The constraint satisfaction problem (CSP) has important
  applications in computer science and AI. In particular,
  infinite-domain CSPs have been intensively used in subareas of AI such as
  spatio-temporal reasoning.
  Since constraint satisfaction is a computationally hard problem, much work has been devoted to 
  identifying restricted problems that are efficiently solvable. 
  One way of doing this is to restrict the interactions of variables and constraints, and a
  highly successful approach is to bound the treewidth of
  the underlying primal graph.
  Bodirsky \& Dalmau [{\em J. Comput. System. Sci.} 79(1), 2013] and Huang et al. [{\em Artif. Intell.} 195, 2013] proved that
  CSP$(\Gamma)$ can be solved in $n^{f(w)}$ time (where~$n$ is the size of
  the instance, $w$ is the treewidth of the primal graph and $f$ is a computable function) 
  for certain classes of constraint languages~$\Gamma$. We improve this bound to $f(w) \cdot n^{O(1)}$,
  where the function $f$ only depends on the language $\Gamma$,
  for CSPs whose basic relations have the patchwork property. 
  Hence, such problems are fixed-parameter tractable and our algorithm is
  asymptotically faster than the previous ones. Additionally, our approach
  is not restricted to binary constraints, so it is applicable to a strictly 
  larger class of problems than that of Huang et al.
  However, there exist natural problems
  that are covered by Bodirsky \& Dalmau's algorithm but not by ours, and we begin investigating
  ways of generalising our results to larger families of languages. We also analyse our algorithm
  with respect to its running time and show that it is optimal (under the Exponential Time Hypothesis) for certain languages such as Allen's Interval Algebra.
\end{abstract}


\end{frontmatter}

\section{Introduction}

The constraint satisfaction problem over a constraint language $\Gamma$ (CSP$(\Gamma)$) is the problem of finding
a variable assignment which satisfies a set of constraints, where each constraint is constructed from a
relation in $\Gamma$. This problem can be used to model
many problems encountered in computer science and AI, see e.g. Rossi et al.~\cite{Rossi:etal:CP} or Dechter~\cite{Dechter:CP}. 
The CSP is computationally hard in
the general case; if the variable domains are finite, then the problem is \NP-complete, and otherwise
it may be of arbitrarily high complexity~\cite{Bodirsky:Grohe:icalp2008}.
Hence, identifying tractable problems is of great practical interest. 

Tractable fragments have historically been identified using two different methods:
either (1) restrict the relations that are allowed in the constraint language or (2) restrict how variables and constraints interact in 
problem instances. We focus on the second kind of restrictions in this article; these are often referred to
as {\em structural restrictions}.
One common way of studying structural restrictions is via the {\em primal graph}: this graph
has the variables as its vertices with two of them joined by an edge if they occur together in the scope
of a constraint. The graph parameter treewidth~\cite{Bertele:Brioschi:NDP72,Robertson:Seymour:jctb84} 
has proven to be very useful in this context, since many \NP-hard graph problems 
are tractable on instances with bounded treewidth.
The treewidth of the primal graph has been extensively used in the study of finite-domain CSPs. It is known
that the problem is {\em fixed-parameter tractable} (fpt), i.e.
it can be solved in $f(w+d) \cdot n^{O(1)}$ time, where $n$ is the size of the instance, $w$ is the
treewidth of the primal graph, $d$~is the domain size and $f$ is some computable
function. This was proven by Gottlob et al.~\cite{Gottlob:etal:ai2002};
see also Samer \& Szeider~\cite{Samer:Szeider:jcss2010}
for a more general treatment.

Let us now consider infinite-domain CSPs. For certain classes of constraint languages,
Bodirsky \& Dalmau~\cite[Corollary~1]{Bodirsky:Dalmau:jcss2013} proved that
CSP$(\Gamma)$ can be solved in $n^{O(w)}$ time (where the exact expression in the $O(w)$
term may depend on the constraint language), while
Huang et al.~\cite[Theorem~6]{Huang:etal:ai2013} obtained the bound
$O(w^3n \cdot {\rm e}^{w^2 \log n}) = n^{O(w^2)}$.
These results prove the weaker property of membership in the complexity class \XP.
Algorithms with a running time bounded by $n^{f(w)}$ are obviously polynomial-time when $w$ is fixed. 
However, since $w$ appears in the exponent, such algorithms become impractical (even for small $w$) when large instances are considered. 
It is significantly better if a problem is fpt and can be solved in time $f(w) \cdot n^{O(1)}$,
since the order of the polynomial in $n$ does not depend at all on $w$. 

Our main result is an fpt algorithm for CSPs where the underlying basic relations have the 
{\em patchwork property}~\cite{Lutz:Milicic:jar2007}. Several important CSPs (such as Allen's Interval Algebra
and RCC8) are known to have this property. The patchwork property ensures that the union of two
satisfiable instances of the CSP, whose constraints agree on their common
variables, is also satisfiable. With the above discussion in mind, it is clear
that our algorithm has better computational properties than the two previous ones.
We will now briefly compare the applicability of the algorithms; more information on this can be found
in Section~\ref{sec:discussion}.
Bodirsky \& Dalmau's algorithm (BD) works for constraint languages
that are $\omega$-{\em categorical} (in fact, it even works for languages where only the
core is $\omega$-categorical), while Huang, Li \& Renz's algorithm (HLR) works for languages with binary relations that have
the {\em atomic network amalgamation property} (aNAP).
Our algorithm has a wider applicability than HLR,
since aNAP implies the patchwork property and our algorithm is
not restricted to binary relations.
The relation to BD is more complex, since there are problems 
that are covered by BD but not by our algorithm.
In some cases our algorithm is not applicable to a CSP directly,
but one can find an equivalent problem that has the patchwork property via {\em homogenisation} --- 
one example is the {\em Branching Time Algebra}~\cite{Anger:etal:spie91} --- the details are discussed later on.
However, there are cases where homogenisation is not applicable, 
so the exact dividing line is unfortunately unclear. 

In Section~\ref{sec:prelim}, we introduce some necessary preliminaries.
The remainder of this article is divided into three distinct parts. In the first part 
(Sections~\ref{sec:results} and~\ref{sec:lowerbounds}), we present our main algorithm
and prove that it achieves the required time bound.
Our algorithm is based on dynamic programming and it is quite different from the BD and HLR algorithms: BD is based on a transformation to Datalog, while HLR is a recursive algorithm based on ideas by Darwiche~\cite{Darwiche:ai2001}.
We complement our algorithmic results with tight lower bounds based on the Exponential Time Hypothesis:
these show that the existence of significantly faster algorithms is not possible in certain special cases such as Allen's Interval Algebra.

In the second part (Section~\ref{sec:applicability}), we analyse the applicability of our algorithm. Even though the
patchwork property is well known within the CSP community, there are not many
formalisms that have been proven to have this property.
By using certain model-theoretical concepts,
we obtain an alternative way of identifying
constraint languages with the patchwork property. Based on this, 
we demonstrate how to apply our results
on constraint languages that are definable in $({\rationals};<)$ (with
applications in, for instance, temporal reasoning and scheduling) and phylogeny languages
(which are useful in bioinformatics).
These classes of languages give rise to CSPs with non-binary relations. Such relations
have, unfortunately, not been well studied in AI,
since the focus has almost exclusively been on binary relations. 
The phylogeny languages
are particularly interesting in this respect, since the basic relations themselves
are non-binary.

In the third part (Section~\ref{sec:beyondPP}), we study how our fpt result can be generalised to languages that do not have the
patchwork property. One concrete example of an interesting language without the patchwork property is 
the previously mentioned Branching Time Algebra. We use model-theoretic tools to achieve
one possible generalisation:
we show that {\em homogenisation} can be used to extend our fpt result to certain classes of languages
that do not have the patchwork property, and this extended result covers the BTA.

\newpage
\section{Preliminaries}
\label{sec:prelim}

In this section we introduce the necessary prerequisites.

\subsection{Relational Structures}

A {\em (relational) signature} $\tau$ is a set of symbols, 
each with an associated natural number called their {\em arity}. 
A {\em (relational) $\tau$-structure} $\A$ consists of
a set $D$ (the domain), together with relations $R^\A \subseteq D^k$ 
for each $k$-ary symbol $R \in \tau$. 
A structure is {\em countable} if its domain is a countable set.

Let $\A$ be a $\tau$-structure over a domain $D$.
We say that $\A$ is $k$-{\em ary} if every relation in $\A$ has arity $k$.
Define $\A^{=k} = \bigcup \{ R \in \A \mid R \textnormal{ has arity } k \}$,
i.e. $\A^{=k}$ is the union of all $k$-ary relations in $\A$.

The relations in $\A$ are {\em jointly exhaustive} (JE)
if for all $k \geq 2$, 
$\A^{=k}$ is either empty or equal to $D^k$.
They are {\em pairwise disjoint} (PD)
if $R \cap R' = \varnothing$ for all distinct $R,R' \in \A$.
In other words, the relations are JEPD
if the nonempty subsets $\A^{=k}$ partition $D^k$.

The following concept is also used in the literature:
the relations of $\A$ are JE$^{+}$ if 
there is a natural number $d \geq 2$ such that $\A^{=k} = D^k$
for all $k\in \{2,\ldots,d\}$, and $\A^{=k}$ is empty otherwise.
Note that any set of JE relations can be augmented 
with the total relations $D^k$ for all $k\in \{2,\ldots,d\}$
where $\A^{=k}$ is empty, thus making the (trivially) extended structure JE$^+$.
Also note that every JE$^+$ structure is also JE, and that a $k$-ary structure cannot be JE$^+$ unless $k=2$.
The difference between JE and JE$^+$ reflects various traditions within the CSP
community. Some researchers have concentrated on $k$-ary structures and for them
the JE property is natural. Others have concentrated on structures with mixed
arities and then the JE$^+$ property becomes natural.

Denote the equality relation over the domain $D$ by 
$Eq_2 = \{(d,d) \mid d \in D\}$.
The relations of $\A$ are {\em jointly diagonalizable} (JD)
if $\bigcup \{ R \in \A \mid R \subseteq Eq_2 \} = Eq_2$.
Note that JD holds vacuously if the equality relation is included in $\A$, and that
a $k$-ary structure can only be JD if $k=2$.

\subsection{Logic}

Let $\A$ be a $\tau$-structure.
First-order formulas $\phi$ over $\A$ (or, for short, $\A$-formulas)
are defined using the logical symbols of universal and existential
quantification, disjunction, conjunction, negation,
equality, bracketing, variable symbols, the relation symbols from $\tau$, and
the symbol $\bot$ for the truth-value false. 
First-order formulas over $\A$ can be used to define relations: 
for a formula $\phi(x_1,\ldots,x_k)$ 
with free variables $x_1,\ldots,x_k$, the corresponding relation $R$
is the set of all $k$-tuples $(t_1,\ldots,t_k) \in D^k$
such that $\phi(t_1,\ldots,t_k)$ is true in $\A$. 
In this case we say that $R$ is {\em first-order definable} in $\A$.
Our definitions are always parameter-free, i.e. we do not allow
the use of domain elements within them.
We may assume without loss of generality that 
all formulas defining relations 
are in disjunctive normal form (DNF). 
A formula is in DNF if it is a disjunction of one or more
conjunctions of one or more {\em atomic formulas} 
of the type $R(\bar{x})$ or $\neg R(\bar{x})$, 
where $R \in \A \cup \{=\}$ and $\bar{x}$ is a sequence of variables.
The conjunctions of atomic formulas are referred to as {\em terms}.

The most common way of using JEPD relations in AI-relevant CSPs 
is via the constraint language $\disj{\A}$, where $\A$ is a $k$-ary structure.
The set $\disj{\A}$ contains the unions of all subsets of $\A$. 
Equivalently, $R \in \disj{\A}$ if $R$ can be written as
a disjunction $R_1(\bar{x}) \vee \cdots \vee R_p(\bar{x})$ where
$R_1,\ldots,R_p \in \A$ and $\bar{x}=(x_1,\ldots,x_k)$.
This definition demands that all relations in $\A$ have the same arity.
Since we want to study more expressive sets of relations, we
let $\boole{\A}$ denote the set of relations that are definable
by quantifier-free formulas that {\em only} contain the relations in~$\A$, i.e.
one is not allowed to use the equality relation $=$ unless it is a member of~$\A$.
When the relations in $\A$ have the same arity,
$\disj{\A} \subsetneq \boole{\A}$ and
$\boole{\A}$ strictly generalises $\disj{\A}$
since all relations in $\disj{\A}$ can be defined by disjunctive formulas.
Note that if the set of relations of $\A$ is finite and JEPD, 
then we may assume that all formulas are negation-free: 
any negated relation can be replaced by
the disjunction of all other relations.

\subsection{Constraint Satisfaction}

Let $\A$ be a $\tau$-structure with domain $D$.
The \textsc{Constraint Satisfaction Problem} 
over $\A$ (CSP$(\A)$) is defined as follows:

\begin{sloppypar}
\smallskip
\noindent
{\sc Instance:} A set $V$ of variables and a set $C$ 
of {\em constraints} of the form
$R(v_1,\ldots,v_k)$, where $R \in \A$ is a relation of arity $k$ and
$v_1,\ldots,v_k \in V$. \\
{\sc Question:} Is there an assignment $f:V \rightarrow D$ such that
$(f(v_1),\ldots,f(v_k)) \in R$ for every $R(v_1,\ldots,v_k) \in C$?
\smallskip
\end{sloppypar}

The structure $\A$ is often referred to as the {\em constraint language}.
Let $\A$ be a finite constraint language with JEPD relations.
Consider a finite $\G \subseteq \boole{\A}$ and let
$\inst = (V, C)$ be an instance of CSP$(\G)$.
Recall that every relation used in a constraint in $C$
can be defined by a DNF $\A$-formula that involves only 
positive atomic formulas of the form $R(\bar{x})$, where $R \in \A$.
A {\em certificate} for $\inst$ is a satisfiable instance 
$\cert = (V, C')$ of CSP$(\A)$ that {\em implies}
every constraint in $C$, i.e.
for every $R(v_1,\ldots,v_k)$ in $C$,
there is a term in the definition of this constraint
(as a DNF $\A$-formula) such that all constraints 
in this term are in~$C'$.

\begin{proposition}
\label{prop:certif-sat}
An instance of CSP$(\G)$ admits a certificate 
if and only if it is satisfiable.
\end{proposition}

\begin{sloppypar}
Now assume CSP$(\A)$ is decidable.
Then, there is an algorithm deciding whether an instance 
$(V,C')$ of CSP$(\A)$ is a certificate for 
an instance $(V,C)$ of CSP$(\G)$:
first, check that $(V,C')$ is satisfiable,
and then, for all $R(v_1,\ldots,v_k) \in\nobreak C$, verify that 
$(V,C')$ implies $R(v_1,\ldots,v_k)$ by considering every term 
in the definition of $R$ and checking if it is included in $C'$.
Note that the length of the DNF formula defining $R$
depends only on the arity of $R$ and $\abs{\A}$,
which are both bounded by constants since $\G$ and $\A$ are finite languages.
Thus, if CSP$(\A)$ is solvable in polynomial time, then
the certificate test can also be performed in polynomial time.
\end{sloppypar}

An instance of CSP$(\A)$ is {\em complete}
if it contains a constraint over every
$k$-tuple of (not necessarily distinct) variables
for every $k$ such that there is a relation in $\A$ of arity $k$.
A certificate is complete if it is 
a complete instance of CSP$(\A)$.
Since the relations in $\A$ are JE,
any certificate can be extended to a complete one.
Thus, we can assume that all certificates are complete.

\begin{example}
Consider the structure ${\bf Q}=({\rationals};<,>,=)$, 
i.e. the rationals under the natural ordering.
The relation $B=\{(x,y,z) \in {\rationals}^3 \; | \; x<y<z \vee z<y<x\}$ is known
as the {\em betweenness} relation and it is a member of $\boole{{\bf Q}}$. Let
$I=(\{w,x,y,z\} \; | \; \{B(w,x,y),B(x,y,z)\})$ be an instance of
CSP$(\{B\})$. The instance $I$ is satisfiable and this is witnessed
by the solution $f(w)=0,f(x)=1,f(y)=2,f(z)=3$.
A certificate for this instance is
$\{w<x,x<y,y<z\}$. A complete certificate is

\[\begin{array}{llll}
w=w, & w<x, & w<y, & w<z, \\
x>w, & x=x, & x<y, & x<z, \\
y>w, & y>x, & y=y, & y<z, \\
z>w, & z>x, & z>y, & z=z. \\
\end{array}
\]

\end{example}

For any instance $\inst = (V, C)$ of CSP
and any set of variables $U \subseteq V$, 
define $C[U] \subseteq C$ 
to include all constraints whose scope is in $U$.
We say that $\inst[U] = (U, C[U])$ 
is the {\em subinstance} of $\inst$ induced by $U$.
We will also say that $\inst[U]$ is obtained by
{\em projecting $\inst$ onto $U$}.
Properties of certificates (including completeness) 
are preserved under projections.
We formalise this observation below.
\begin{proposition}
\label{prop:certif-projection}
If $\cert$ is a certificate for $\inst = (V, C)$,
then $\cert[U]$ is a certificate for $\inst[U]$ 
for all $U \subseteq V$. If $\cert$ is complete, then
$\cert[U]$ is also complete.
\end{proposition}

\subsection{Parameterized Complexity}

In parameterized 
algorithmics~\cite{DowneyFellows13,book/FlumG06,book/Niedermeier06}
the runtime of an algorithm is studied with respect to
the input size~$n$ and a parameter $p \in \naturals$.
The basic idea is to find a parameter that describes the structure of
the instance such that the combinatorial explosion can be confined to
this parameter.
In this respect, the most favourable complexity class is \FPT
(\emph{fixed-parameter tractable}),
which contains all problems that can be decided by an algorithm
running in $f(p)\cdot n^{O(1)}$ time, where $f$ is a computable
function.
Problems that can be solved in this time are said to be 
\emph{fixed-parameter tractable} (fpt).
The more general class \XP contains all problems decidable
in $n^{f(p)}$ time, i.e. the problems solvable in polynomial time
when the parameter $p$ is bounded.
Clearly, $\FPT \subseteq \XP$.
Moreover, the inclusion is strict 
(see e.g.~\cite{book/FlumG06}).

\begin{sloppypar}
We will concentrate on one well-known parameter in this article: the {\em treewidth} of the
{\em primal graph}. Thus, if we state that some problem is fpt, then
we always mean with respect to this parameter. The primal graph of an instance of a CSP
is the undirected graph whose vertices 
coincide with the variables of the instance, and where
two vertices are joined by an edge 
if they occur in the scope of the same constraint. Treewidth
is based on {\em tree decompositions}:
a tree decomposition $(T, X)$ of an undirected graph $G = (V, E)$ consists 
of a rooted tree $T$ and a mapping $X$ from the nodes of $T$ to the subsets of $V$.
The subsets $X(t)$ are called {\em bags}.
$T_t$ stands for the subtree rooted at $t$,
while $V_t$ is the set of all variables occurring in the bags of $T_t$, i.e.
$V_t = \bigcup_{s \in V(T_t)} X(s)$.
A tree decomposition fulfils the following properties:
\end{sloppypar}

\newpage

\begin{enumerate}
    \item \label{def:treedecomp:t1}
    $\bigcup_{t \in V(T)} X(t) = V$.
    \item \label{def:treedecomp:t2}
    If $(u,v) \in E$, then $u,v \in X(t)$ for some $t \in V(T)$.
    \item \label{def:treedecomp:t3}
    For any $t_1, t_2, t_3 \in T$, if $t_2$ lies on the path
    between $t_1$ and $t_3$, then $X(t_1) \cap X(t_3) \subseteq X(t_2)$.
\end{enumerate}

The width of a tree decomposition $T$ is defined as $\max \{ \abs{X(t)} : t \in V(T) \} - 1$.
The treewidth of a graph $G$ is the minimum width
of a tree decomposition of~$G$.
It is \NP-complete to determine whether a graph has treewidth at most $k$ \cite{Arnborg:etal:sijmaa87},
but when $k$ is fixed the graphs with treewidth~$k$ can be recognised 
and corresponding tree decompositions can be constructed in linear time~\cite{Bodlaender:sicomp96}.

\subsection{Qualitative Spatial and Temporal Reasoning}
\label{sec:qstr}

We will consider several well-known formalisms 
for qualitative spatial and temporal reasoning.
All of them can be defined as $\disj{\B}$ 
via a binary constraint language $\B$ 
with JEPD relations. It is important
to note that the exact choice of relations for representing
a reasoning problem as a CSP may be crucial. This is most
easily illustrated with the RCC5 formalism that is introduced
in Item~\ref{item:4-rcc} below. RCC5 can be represented with structures~${\bf A}$ and ${\bf B}$ such that CSP$({\bf A})$ is the
same computational problem as CSP$({\bf B})$,
while~${\bf A}$ and ${\bf B}$ are very different from a
model-theoretical point of view.
Bodirsky and Jonsson discuss this in some detail for RCC5 in~\cite[Section~2.5.2]{Bodirsky:Jonsson:jair2017}.
They also discuss that there are ${\bf A}$ and ${\bf B}$
that look like suitable representations of RCC5 
(for instance, by having the ``right'' composition tables),
but have different CSPs. For RCC5 and RCC8, we will thus exclusively use the
representations suggested by Bodirsky \& Wölfl~\cite{Bodirsky:Wolfl:ijcai2011},
whose CSP coincides with the standard interpretation of RCC relations. 
We will come back to the importance of choosing the right representation
in Sections~\ref{sec:applicability} and~\ref{sec:beyondPP}.

The choice of representation for the formalisms in Items~\ref{item:1-aia}--\ref{item:3-cdc} is, fortunately, much easier:
the natural representations via concrete objects in ${\rationals}^d$
have proven to capture the intended computational problems and at the same time
have advantageous model-theoretical properties.
We will consequently use these representations throughout the article.

\begin{enumerate}

\item\label{item:1-aia}
{\em Allen's Interval Algebra}~(IA)~\cite{Allen:cacm83} is
a temporal reasoning formalism where one considers relations between
intervals of the form $I = [I^{-},I^{+}]$, 
where $I^{-},I^{+} \in {\rationals}$, $I^{-} < I^{+}$ are 
the start and end points, respectively.
The language $\IAlang$ consists of thirteen 
basic relations illustrated in Table~\ref{tb:allen}.

\begin{table}[htb]
    \begin{center}
      {\scriptsize
    \begin{tabular}{|ll|c|l|}\hline
              Basic relation        &               & Example                   & Endpoints\\
        \hline\hline
              $I$ precedes      $J$ & ${\sf p}$     & \stexttt{iii\spv}         & $I^{+}<J^{-}$
         \\ \cline{1-2}
              $J$ preceded-by   $I$ & ${\sf pi}$    & \stexttt{\spv jjj}        & \\
        \hline
              $I$ meets         $J$ & ${\sf m}$     & \stexttt{iiii\spiv}       & $I^{+}=J^{-}$
        \\ \cline{1-2}
              $J$ met-by        $I$ & ${\sf mi}$    & \stexttt{\spiv jjjj}      & \\
        \hline
              $I$ overlaps      $J$ & ${\sf o}$     & \stexttt{iiii\spii}       & $I^{-}<J^{-}<I^{+}<J^{+}$ \\ \cline{1-2}
              $J$ overlapped-by $I$     & ${\sf oi}$    & \stexttt{\spii jjjj}      & 
        \\ \hline
              $I$ during        $J$ & ${\sf d}$     & \stexttt{\spii iii\spii}  & $I^{-}>J^{-}$, \\ \cline{1-2}
              $J$ includes      $I$ & ${\sf di}$    & \stexttt{jjjjjjj}         & $I^{+}<J^{+}$ \\
        \hline
              $I$ starts        $J$ & ${\sf s}$     & \stexttt{iii\spiv}        & $I^{-}=J^{-}$,
        \\ \cline{1-2}
              $J$ started-by    $I$ & ${\sf si}$    & \stexttt{jjjjjjj}         & $I^{+}<J^{+}$ \\
        \hline
              $I$ finishes      $J$ & ${\sf f}$     & \stexttt{\spiv iii}       & $I^{+}=J^{+}$,
        \\ \cline{1-2}
              $J$ finished-by   $I$ & ${\sf fi}$    & \stexttt{jjjjjjj}         & $I^{-}>J^{-}$ \\
        \hline
              $I$ equals        $J$ & ${\sf e}$     & \stexttt{iiii}            & $I^{-}=J^{-}$, \\
                                    &               & \stexttt{jjjj}            & $I^{+}=J^{+}$ \\
\hline
    \end{tabular}}
    \caption{The thirteen basic relations in Allen's Interval Algebra. The
    endpoint relations $I^- < I^+$ and $J^- < J^+$ that are valid for
    all intervals $I$ and $J$ are omitted.}
    \label{tb:allen}
    \end{center}
\end{table}

\item\label{item:2-ba}
The $d$-dimensional {\em Block Algebra}~(BA$_d$)~\cite{Balbiani:et:al:kr98}
is a generalisation of IA to $d$-dimensional
boxes with sides parallel to the coordinate axes.
The relations in $\B_{\rm BA}$ are $d$-tuples
of IA relations, each one applied in the
corresponding dimension.

\newpage
\item\label{item:3-cdc}
The {\em Cardinal Direction Calculus}~(CDC)~\cite{Ligozat:vlc98}
is a formalism for spatial reasoning
with points on the plane as the basic objects.
The relations in $\CDClang$ correspond to eight
cardinal directions (North, East, South, West and 
four intermediate ones)
plus the equality relation.
They can be viewed as pairs $(R_1, R_2)$
for all choices of $R_1, R_2 \in \{<,=,>\}$,
where each relation applies 
to the corresponding coordinate.
See Table~\ref{tb:cdc} for the correspondence
between cardinal directions and pairs $(R_1, R_2)$.

\begin{table}[b]
    \centering
\scalebox{0.94}{ 
    \begin{tabular}{|c|c|c|c|c|c|c|c|c|}
        \hline
        {\rm =} & {\rm N} & {\rm E} & {\rm S} & {\rm W} & {\rm NE} & {\rm SE} & {\rm SW} & {\rm NW} \\
        \hline
        $(=,=)$ & $(=,>)$ & $(>,=)$ & $(=,<)$ & $(<,=)$ & $(>,>)$  & $(>,<)$  & $(<,<)$  & $(<,>)$  \\
        \hline
    \end{tabular}
}
    \caption{The basic relations of Cardinal Direction Calculus}
    \label{tb:cdc}
\end{table}

\item\label{item:4-rcc}
The {\em Region Connection Calculus}~(RCC8)~\cite{Randell:etal:kr92} is a formalism for qualitative spatial reasoning,
where the basic objects (referred to as {\em regions}) are non-empty regular closed subsets of a topological space.
The regions do not have
to be internally connected, that is, they may consist of
different disconnected pieces. 
$\B_{\rm RCC8}$ contains eight relations:
{\sf EQ} (equal),
{\sf PO} (partial overlap), 
{\sf DC} (disconnected),
{\sf EC} (externally connected),
{\sf NTPP} (non-tangential proper part),
its converse ${\sf NTPP}^{-1}$,
{\sf TPP} (tangential proper part)
and its converse ${\sf TPP}^{-1}$.
See Figure~\ref{fig:rcc8} for examples.
RCC5 is a variant of RCC8 where 
one is not able to distinguish regions from their topological closure, i.e.
the distinction between boundary points and interior points is ignored.
The disconnectedness relations {\sf DC} and {\sf EC}
are replaced by ${\sf DR}={\sf DC} \cup {\sf EC}$ (distinct from),
the tangential and non-tangential proper part relations {\sf TPP} and {\sf NTPP}
are replaced by ${\sf PP}={\sf TPP} \cup {\sf NTPP}$ (proper part), and {\sf PP}$^{-1}$ is defined analogously.
\end{enumerate}

\def\rccfigwidth{0.2\textwidth}
\begin{figure}[h]
    \centering
    \captionsetup[subfigure]{labelformat=empty}
    \begin{subfigure}[b]{\rccfigwidth}
        \centering
        \begin{tikzpicture}[scale=0.75]
            \draw (0,0) circle (1);
            \draw[dashed] (0,0) circle (1);
            \draw (-0.33,0) node {$X$};
            \draw (0.33,0) node {$Y$};
        \end{tikzpicture}
        \caption{${\sf EQ}(X,Y)$}
    \end{subfigure}
    \hfill
    \begin{subfigure}[b]{\rccfigwidth}
        \centering
        \begin{tikzpicture}[scale=0.5]
            \draw (0,1.1) circle (1);
            \draw (0,1.1) node {$X$};
            \draw[dashed] (0,-1.1) circle (1);
            \draw (0,-1.1) node {$Y$};
        \end{tikzpicture}
        \caption{${\sf DC}(X,Y)$}
    \end{subfigure}
    \hfill
    \begin{subfigure}[b]{\rccfigwidth}
        \centering
        \begin{tikzpicture}
            \draw (-0.3,0) circle (0.5);
            \draw (-0.3,0) node {$X$};
            \draw[dashed] (0,0) circle (1);
            \draw (0.5,0) node {$Y$};
        \end{tikzpicture}
        \caption{${\sf NTPP}(X,Y)$}
    \end{subfigure}
    \hfill
    \begin{subfigure}[b]{\rccfigwidth}
        \centering
        \begin{tikzpicture}
            \draw[dashed] (-0.3,0) circle (0.5);
            \draw (-0.3,0) node {$Y$};
            \draw (0,0) circle (1);
            \draw (0.5,0) node {$X$};
        \end{tikzpicture}
        \caption{${\sf NTPP}^{-1}(X,Y)$}
    \end{subfigure}

    \begin{subfigure}[b]{\rccfigwidth}
        \centering
        \begin{tikzpicture}[scale=0.5]
            \draw (0,0.75) circle (1);
            \draw[dashed] (0,-0.75) circle (1);
            \draw (0,0.75) node {$X$};
            \draw (0,-0.75) node {$Y$};
        \end{tikzpicture}
        \caption{${\sf PO}(X,Y)$}
    \end{subfigure}
    \hfill
    \begin{subfigure}[b]{\rccfigwidth}
        \centering
        \begin{tikzpicture}[scale=0.5]
            \draw (0,1) circle (1);
            \draw (0,1) node {$X$};
            \draw[dashed] (0,-1) circle (1);
            \draw (0,-1) node {$Y$};
        \end{tikzpicture}
        \caption{${\sf EC}(X,Y)$}
    \end{subfigure}
    \hfill
    \begin{subfigure}[b]{\rccfigwidth}
        \centering
        \begin{tikzpicture}
            \draw (-0.5,0) circle (0.5);
            \draw (-0.5,0) node {$X$};
            \draw[dashed] (0,0) circle (1);
            \draw (0.5,0) node {$Y$};
        \end{tikzpicture}
        \caption{${\sf TPP}(X,Y)$}
    \end{subfigure}
    \hfill
    \begin{subfigure}[b]{\rccfigwidth}
        \centering
        \begin{tikzpicture}
            \draw[dashed] (-0.5,0) circle (0.5);
            \draw (-0.5,0) node {$Y$};
            \draw (0,0) circle (1);
            \draw (0.5,0) node {$X$};
        \end{tikzpicture}
        \caption{${\sf TPP}^{-1}(X,Y)$}
    \end{subfigure}
    
    \caption{Illustration of the relations of RCC8 with 
    two-dimensional disks.}
    \label{fig:rcc8}
\end{figure}
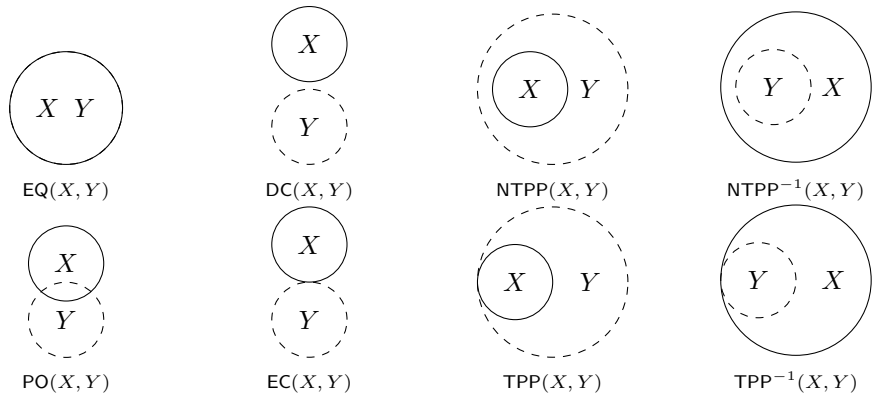

\section{The Main Algorithm}
\label{sec:results}

The goal of this section is to present an fpt algorithm that is
applicable to a wide range of interesting CSPs.
The basic CSP property underlying our algorithm is the following.

\begin{definition}[Lutz and Mili\v{c}i\'{c}~\cite{Lutz:Milicic:jar2007}] \label{def:patchwork}
A JEPD constraint language $\A$ 
has the {\em patchwork property (PP)}
if, for every pair of complete satisfiable instances
$\inst_1 = (V_1,C_1)$ and $\inst_2 = (V_2,C_2)$ of CSP$(\A)$
such that $\inst_1[V_1 \cap V_2] = \inst_2[V_1 \cap V_2]$,
the instance $(V_1 \cup V_2, C_1 \cup C_2)$ is also satisfiable.
\end{definition}

\begin{sloppypar}
We want to underline the importance of the completeness condition
in the previous definition:
for example, consider the JEPD constraint language
$(<\nobreak ,=\nobreak ,>\nobreak )$ with domain $\rationals$
and the two satisfiable incomplete instances
$(\{a,x,b\}, \{a<x, x<b\})$ and $(\{a,y,b\}, \{a>y, y>b\})$.
The intersection of these instances contains no constraints, so it is trivially satisfiable.
However, their union is not satisfiable since the constraints
imply that $a < b$ and $a > b$ hold simultaneously.
\end{sloppypar}

Several prominent formalisms for qualitative spatial and temporal
reasoning in AI are known to have the patchwork property.
For example, the JEPD basic relations of Allen's Interval Algebra,
the Block Algebra, and the Cardinal Direction Calculus
have the patchwork property~\cite{Huang:kr2012,Lutz:Milicic:jar2007},
assuming that the standard representations from Section~\ref{sec:qstr} are used.
The picture is slightly more complex for RCC8 and RCC5.
Lutz and Mili\v{c}i\'{c}~\cite{Lutz:Milicic:jar2007} show that RCC8 restricted to the
real plane has the patchwork property, and Huang~\cite{Huang:kr2012} points out that this result
can be lifted to the multi-dimensional case via Bodirsky \& Wölfl's~\cite{Bodirsky:Wolfl:ijcai2011} representation.
Baader \& Rydval~\cite{Baader:Rydval:ijcar2020} also point this out
in a more general setting; we will come back to their results in Section~\ref{sec:applicability}.

Before we present the algorithm, we state a
lemma that is a direct consequence of the patchwork property.

\begin{lemma} \label{lem:patchwork-jepd}
Let $\A$ be a finite set of 
JEPD relations with the patchwork property
and assume that $\G \subseteq \boole{\A}$ is finite.
For any two satisfiable instances
$\inst_1 = (V_1,C_1)$ and $\inst_2 = (V_2,C_2)$ 
of CSP$(\G)$
admitting complete
certificates $\cert_1$ and $\cert_2$
such that $\cert_1[V_1 \cap V_2] = \cert_2[V_1 \cap V_2]$,
the instance $(V_1 \cup V_2, C_1 \cup C_2)$ 
is also satisfiable.
\end{lemma}

\begin{proof}
Let $\cert_1 = (V_1,C'_1)$, $\cert_2 = (V_2,C'_2)$, and
define the instance $\cert_{\cup}=(V_1 \cup V_2, C'_1 \cup C'_2)$ of CSP$(\A)$.
We claim that $\cert_{\cup}$ is a certificate for 
$(V_1 \cup V_2, C_1 \cup C_2)$.
First, we note that $\cert_{\cup}$ is satisfiable by Definition~\ref{def:patchwork}.
Now consider a constraint $c \in C_1 \cup C_2$.
Then $c$ is in one of the following sets:
$C_1 \setminus C_2$, $C_1 \setminus C_2$ or $C_1 \cap C_2$.
In the first case, $\cert_{\cup}[V_1] = \cert_1$ implies $c$.
Similarly, in the second case $\cert_{\cup}[V_2] = \cert_2$ implies $c$.
Finally, if $c$ is in the intersection of $C_1$ and $C_2$,
then $\cert_{\cup}[V_1 \cap V_2] = \cert_1[V_1 \cap V_2] = \cert_2[V_1 \cap V_2]$ implies $c$.
Thus, $\cert_{\cup}$ implies $(V_1 \cup V_2, C_1 \cup C_2)$.
\end{proof}

To simplify the presentation, we will use
a particular kind of tree decomposition.
A tree decomposition is {\em nice} if it fulfils the following properties:
\begin{itemize}
    \item 
    $X(r) = \varnothing$ and $X(\ell) = \varnothing$
    for the root $r$ and all leaf nodes $\ell$ in $T$.
    \item
    Every non-leaf node in $T$ is one of the following types:
    \begin{itemize}
        \item An {\bf introduce node}: a node $t$ with exactly one child $t'$
        such that $X(t) = X(t') \cup \{v\}$ for some
 $v \in V\setminus X(t')$.
        \item A {\bf forget node}: a node $t$ with exactly one child $t'$
        such that $X(t) = X(t') \setminus \{w\}$ for some
$w \in V \cap X(t')$.
        \item A {\bf join node}: a node $t$ with exactly two children $t_1$ and $t_2$
        such that $X(t) = X(t_1) = X(t_2)$.
    \end{itemize}
\end{itemize}
Given a tree decomposition $T$ of an $n$-vertex graph,
one can construct a nice tree decomposition of the same width 
and with $O(n)$ nodes in linear time~\cite{bodlaender1996efficient}. 

We are now ready to present the fpt algorithm.

\begin{theorem}
\label{thm:fpt}
Let $\A$ be a finite
constraint language with 
JEPD relations and the patchwork property.
Assume CSP$(\A)$ is decidable.
For any finite constraint language 
$\G \subseteq \boole{\A}$, CSP$(\G)$ is fpt
parameterized by the treewidth of the primal graph.
\end{theorem}
\begin{proof}
Let $\inst = (V, C)$ be an instance of CSP$(\G)$ and assume
$(T,X)$ is a nice tree decomposition of its primal graph.
The algorithm works as follows:
for every node $t \in T$, we compute the set $R(t)$
consisting of all certificates for~$\inst[V_t]$
projected onto $X(t)$.
Clearly, $\inst$ is satisfiable if and only if 
$R(r) \neq \varnothing$, where~$r$ is the root of $T$.
We compute $R(t)$ using dynamic programming from the leaves upwards,
i.e. a node is processed only if all its children have already been processed.

To start, we set $R(\ell) = \{ (\varnothing, \varnothing) \}$ for all leaf nodes $\ell \in T$.
Since the decomposition is nice, we only need to consider three cases.
If $t$ is an introduce node with a child $t'$,
we enumerate certificates $\cert$ for $\inst[X(t)]$
and add $\cert$ to $R(t)$ if $\cert[X(t')]$ is in $R(t')$.
If $t$ forgets a variable $w$ and has a child $t'$,
then $R(t)$ is obtained by enumerating certificates 
in $R(t')$ and removing $w$ together with all constraints 
involving it from the certificate.
Finally, if $t$ joins nodes $t_1$ and $t_2$,
then set $R(t) = R(t_1) \cap R(t_2)$
(recall that we may assume the certificates $\cert$ that we consider are complete).

To show the correctness of the algorithm, 
we prove the following claim for every $t \in T$.

\begin{claim} \label{cl:correctness}
  $\cert$ is a certificate for $\inst[V_t]$ if and only if 
  $\cert[X(t)] \in R(t)$.
\end{claim}
We prove the claim by induction.
In the base case, $R(\ell) = \{ (\varnothing,\varnothing) \}$ 
is indeed the set of all certificates for
$\inst[V_\ell]$ 
for all leaves $\ell$ in $T$, since $V_\ell = X(\ell) = \varnothing$.

If $t$ is an introduce node with child $t'$, 
consider a certificate $\cert$ for $\inst[V_t]$.
Note that $\cert[V_{t'}]$ is a certificate for $\inst[V_{t'}]$,
so $\inst[X(t')] \in R(t')$ by the inductive hypothesis.
Furthermore, $\cert[X(t)]$ is a certificate for $\inst[X(t)]$, 
thus the algorithm adds it to $R(t)$. 
In the opposite direction, consider ${\cal K} \in R(t)$ and observe that,
by construction, there is a certificate $\cert'$ for $\inst[V_{t'}]$ 
such that ${\cal K}[X(t')] = \cert'[X(t')]$.
Since ${\cal K}$ is a certificate for $X(t)$ and 
$X(t) \cap V_{t'} = X(t')$, the union of ${\cal K}$ and $\cert'$
is a certificate for $\inst[V_t]$ by
the patchwork property,
and ${\cal K}$ is precisely its projection onto $X(t)$.

If $t$ is a forget node with a child $t'$,
consider a certificate $\cert$ for $\inst[V_{t'}]$
and note that, since $V_t = V_{t'}$, it is also
a certificate for $\inst[V_t]$.
By the inductive hypothesis, $\cert[X(t')] \in R(t')$,
hence, the algorithm adds $\cert[X(t)]$ to $R(t)$.
In the opposite direction, consider 
${\cal K}[X(t)] \in R(t)$ derived from 
${\cal K} \in R(t')$ and note that the inductive
hypothesis implies that ${\cal K}$ and, subsequently,
${\cal K}[X(t)]$ are projections of a certificate for $\inst[V_t]$. 

If $t$ joins nodes $t_1$ and $t_2$, 
consider a certificate $\cert$ for $\inst[V_t]$.
Note that it is also a certificate for $\inst[V_{t_1}]$ and $\inst[V_{t_2}]$,
since $V_{t_1}, V_{t_2} \subseteq V_t$. 
By the inductive hypothesis, 
$\cert[X(t_1)] = \cert[X(t_2)] = \cert[X(t)] \in R(t_1) \cap R(t_2)$
and the algorithm adds it to $R(t)$.
In the opposite direction, consider ${\cal K} \in R(t) = R(t_1) \cap R(t_2)$.
By the inductive hypothesis, there are certificates 
$\cert_1$ for $\inst[V_{t_1}]$ and $\cert_2$ for $\inst[V_{t_2}]$
such that $\cert_1[X(t)] = \cert_2[X(t)] = {\cal K}$.
By the third property of tree decompositions,
$V_{t_1} \cap V_{t_2} \subseteq X(t)$,
thus the union of $\cert_1$ and $\cert_2$ is 
a certificate for $\inst[V_t]$ by
Lemma~\ref{lem:patchwork-jepd},
and ${\cal K}$ is precisely its projection onto $X(t)$.

\medskip

We continue with the time complexity of the algorithm.
Let $w$ denote the width of the decomposition $(T, X)$,
let~$k$ denote the maximum arity of relations in $\A$,
and assume that $\tau(m)$ is the time required to enumerate
certificates for an instance of CSP$(\G)$ with $m$ variables.
Note that since $\A$ and $\G$ are finite, the function $\tau$
depends only on the number of variables.
Furthermore, $\tau(m)$ is an upper bound on the number of 
complete satisfiable instances of CSP$(\A)$ with $m$ variables.

\begin{claim} \label{cl:runtime}
  For every $t \in T$, the computation of $R(t)$ requires at most 
  $\tau(w+1)^2 \cdot w^{O(k)}$ time.
\end{claim}

First, note that $\tau(\abs{X(t)})$ is an upper bound on $\abs{R(t)}$.
Furthermore, taking a projection of a certificate onto $U \subseteq V$ requires
$\abs{U}^{O(k)}$ time.
If $t$ is an introduce node, the computation of $R(t)$ requires at most
\[\tau(w+1) \abs{R(t')} \cdot \abs{X(t')}^{O(k)} \leq \tau(w+1)^2 \cdot w^{O(k)}\]
time.
If $t$ is a forget node, the computation requires at most 
\[\abs{R(t')} \cdot \abs{X(t)}^{O(k)} \leq \tau(w+1) \cdot w^{O(k)}\]
time.
Finally, if $t$ joins nodes $t_1$ and $t_2$, the computation of $R(t)$
takes at most 
\[\abs{R(t_1)} \abs{R(t_2)} \cdot 
O(\abs{X(t)}^{k+1}) \leq \tau(w+1)^2 \cdot w^{O(k)}\]
time, where 
$w^{O(k)}$ 
accounts for the comparison of a pair of certificates.

There are $O(n)$ nodes in the tree $T$, so the algorithm solves CSP$(\G)$ 
in 
$\tau(w+1)^2 \cdot w^{O(k)} \cdot O(n)$ 
time.
The term 
$\tau(w+1)^2 \cdot w^{O(k)}$
depends only on the parameter $w$, 
hence CSP$(\G)$ is fpt.
\end{proof}

We continue by taking a closer look at some classical 
CSPs for qualitative spatial and temporal reasoning.

\begin{corollary}
  \label{cor:rcc,aia,ba,cdc}
  CSP$(\disj{\B})$ is solvable in 
  \begin{enumerate}
      \item $2^{O(w^2)} \cdot O(n)$ time if 
          $\B$ is $\B_{\rm RCC5}$ or $\B_{\rm RCC8}$, 
      \item $2^{O(w \log w)} \cdot O(n)$ time if 
          $\B$ is $\IAlang$, $\BAlang$ or $\CDClang$. 
  \end{enumerate}
\end{corollary}
\begin{proof}
Consider Claim~\ref{cl:runtime} in Theorem~\ref{thm:fpt}.
Since the languages under consideration are JEPD and binary,
the total number of instances of CSP$(\B)$ with $w$ variables 
is $\abs{\B}^{w^2} = 2^{O(w^2)}$, 
since $\abs{\B}$ is constant.
Solving instances of these CSPs takes polynomial time, 
so $\tau(w) = 2^{O(w^2)}$.
This yields the result for RCC5 and RCC8.

For the remaining cases, we need a tighter bound on $\tau(w)$.
We show that the number of complete certificates for these problems is 
at most $2^{O(w \log w)}$.
Observe that an ordered partition of a set $S$ of size $n$
is a surjective function $\pi : S \rightarrow \range{1}{r}$
for some $r \in \range{1}{n}$.
Any two elements of $S$ can be compared with the usual
relations $\{<,=,>\}$ according to the values
assigned to them by $\pi$.
Observe that that there are at most $n^n = 2^{O(n \log{n})}$
ordered partitions of $S$.

Every complete satisfiable instance of CSP$(\IAlang)$
corresponds to a unique ordered partition of the endpoints 
of the intervals~(see e.g.~\cite{Stockman:mscthesis}).
For an instance with $w$ variables (i.e. $2w$ endpoints),
there are at most $2^{O(w \log w)}$ such partitions.
Thus, an instance of CSP$(\IAdisjlang)$ with $w$ variables
admits at most $2^{O(w \log{w})}$ complete certificates.
Given an ordered partition on the endpoints of the intervals,
there is a polynomial-time procedure that recovers the corresponding
complete satisfiable instance of CSP$(\IAlang)$, if one exists:
for every variable, check that its left endpoint 
precedes its right endpoint -- 
if not, then there is no corresponding instance;
otherwise, deduce the relation between every pair of variables
according to the ordered partition of their endpoints.
The last step works since $\IAlang$ is JEPD.
Finally, observe that generating all (unordered)
partitions of a set takes $O(1)$ amortised
time per partition~\cite{Ichiro:jip1984}
and generating all permutations
takes $O(1)$ time per permutation~\cite{Sedgewick:acmcs77}.
Thus, $\tau(w) = 2^{O(w \log w)}$ for CSP$(\IAdisjlang)$.

The Block Algebra BA$_d$ can be viewed as an extension 
of Allen's Interval Algebra to $d$ dimensions, and the complete certificates
correspond to $d$ ordered partitions of the endpoints. 
We have $( (2w)^{2w} )^d = 2^{O(w \log w)}$ since $d$ is fixed, so $\tau(w) = 2^{O(w \log w)}$ for CSP$(\BAdisjlang)$.

Every satisfiable instance of the CSP$(\CDClang)$
corresponds to two ordered partitions -- 
one for the $x$ coordinates and
one for the $y$ coordinates. 
There are $( w^{w} )^2 = 2^{O(w \log w)}$ 
such pairs of partitions,
so $\tau(w) = 2^{O(w \log w)}$ for CSP$(\CDCdisjlang)$.
\end{proof}

We remark that the proof of Corollary~\ref{cor:rcc,aia,ba,cdc} 
also shows that CSP$(\G)$ for any finite $\G \in \boole{\B}$ 
is solvable in $2^{O(w^2)} \cdot O(n)$ time if 
$\B$ is $\B_{\rm RCC5}$ or $\B_{\rm RCC8}$, 
and in $2^{O(w \log w)} \cdot O(n)$ time if 
$\B$ is $\IAlang$, $\BAlang$ or $\CDClang$. 

\smallskip

So far we have considered only finite constraint languages $\G \in \boole{\B}$.
If the language $\G$ is infinite, the representation
of relations becomes problematic since the maximal arity is no longer bounded by a constant. 
In this case, the DNF formulas defining relations may be arbitrarily large. 
The time complexity of checking whether an instance of CSP$(\B)$ is a certificate 
for an instance of CSP$(\G)$ depends on the size of this representation.
To circumvent this difficulty, we consider two possibilities --
an oracle model and a restricted version of CSP, 
where the scope of every constraint contains only distinct variables.
We remark that in both cases the fpt algorithm from Theorem~\ref{thm:fpt} 
solves CSP$(\G)$ even if $\G$ is infinite.

In the oracle model we assume that, given a constraint
$R(\bar{v})$ with $R \in \G$, the time complexity of 
checking whether an instance $\cert$ of CSP$(\B)$ implies 
$R(\bar{v})$ is in polynomial time in the size of $\cert$,
i.e. it is independent of the representation of~$R$.
Clearly, Claim~\ref{cl:runtime} still holds in this case
since the time to compute the record~$R(t)$ only depends on $\abs{X(t)}$.
Hence, CSP$(\G)$ is fpt in the oracle model even if~$\G$ is infinite.

Alternatively, we can restrict CSP$(\G)$ by disallowing 
repeated variables in the scopes of its constraints.
Note that if the language $\B$ is JD, e.g. if $\B$ contains equality,
then this restriction does not affect the expressive power of the language:
one can introduce many copies of a variable $x$ 
by adding constraints $x = x', x' = x'', \ldots$
and use those in place of repeated variables.
When the primal treewidth of an instance of CSP$(\G)$ is bounded by $w$
and the scopes of constraints contain distinct variables,
no constraint can have arity larger than $w+1$.
Thus, the size of the DNF formula defining any relation in this instance
is bounded by a function of $w$ and the computation of 
the record $R(t)$ only depends on $\abs{X(t)}$.

\section{Tight Lower Bounds}
\label{sec:lowerbounds}

By Corollary~\ref{cor:rcc,aia,ba,cdc}, CSPs for 
Cardinal Direction Calculus, Allen's Interval Algebra and Block Algebra
admit algorithms running in $2^{O(w\log{w})} \cdot n^{O(1)}$ time
on instances with primal treewidth $w$.
It is natural to ask whether the dependence on $w$ can be improved.
In this section we provide evidence that such an improvement is improbable
by establishing tight lower bounds on the running time 
of the algorithms for these problems
assuming the Exponential Time Hypothesis (ETH)~\cite{impagliazzo98}.
{\sc 3-Satisfiability} is the problem of deciding 
whether a propositional formula in conjunctive normal form
with at most three literals in each term admits a satisfying assignment.
The ETH is a standard complexity assumption that rules out
the existence of an algorithm for {\sc 3-Satisfiability} 
running in subexponential time, i.e. there is no algorithm
that runs in $2^{o(n)}$ time where $n$
is the number of variables.

We start by showing a reduction from
the following problem to CSP$(\CDCdisjlang)$:

\medskip
\noindent
{\sc Vertex Colouring} 
\smallskip \\
{\sc Instance}: $(G,k)$, where $G$ is an undirected graph with 
vertex set $V(G)$ and edge set $E(G)$,
and $k \in \naturals$ \\
{\sc Question}: Does there exist $\chi : V(G) \rightarrow \range{1}{k}$
such that $\chi(u) \neq \chi(v)$ for all $(u,v) \in E(G)$?
\medskip

Such a mapping $\chi$ is said to be a \emph{proper $k$-colouring} of $G$.
If such $\chi$ exists, we say that $G$ is \emph{$k$-colourable}.
Lokshtanov et al.~\cite{Lokshtanov:etal:sicomp2018} introduced
a framework for showing ETH-based lower bounds on parameterized algorithms
and used it to prove the following theorem:

\begin{theorem} \label{thm:colouring-lowerbound}
  {\sc Vertex Colouring} cannot be solved in 
  $2^{o(w \log{w})} \cdot n^{O(1)}$ time
  on graphs with treewidth $w$ unless the ETH fails. 
\end{theorem}

\noindent
Before we present our reduction, we prove two
well-known useful observations about treewidth.

\begin{lemma} \label{lem:colour-treewidth}
  Graphs with treewidth $w$ are $(w+1)$-colourable.
\end{lemma}
\begin{proof}
The treewidth of a graph $G$ is alternatively characterised as
$\min \{ \omega(H) - 1 \mid G \textnormal{ is a subgraph of }
H, H \textnormal{ is chordal} \}$,
where $\omega(H)$ is the size of the largest clique in $H$
(e.g. see~\cite[Chapter~12.4]{diestel:graphtheory}).
Chordal graphs are perfect graphs and this implies that $H$ is colourable with $\omega(H) = w+1$ colours
(e.g. see~\cite[Chapter~5.5]{diestel:graphtheory}).
Since $G$ is a subgraph of $H$, it is also $(w+1)$-colourable.
\end{proof}

An induced subgraph $H$ of $G$
is a graph such that $V(H) \subseteq V(G)$
and $(u,v) \in E(H)$ if and only if both $(u,v) \in E(G)$ 
and $u,v \in V(H)$.

\begin{lemma} \label{lem:treewidth-add-variables}
  Let $H$ be an induced subgraph of $G$ and
  let~$w$ be the treewidth of~$H$.
  Then the treewidth of $G$ is at most $w + \abs{V(G)} - \abs{V(H)}$.
\end{lemma}
\begin{proof}
Denote $V(G) \setminus V(H)$ by $U$.
Given a tree decomposition $(T,X)$ of $H$ of width $w$,
observe that adding all variables in $U$ to 
every bag $X(t)$ yields a tree decomposition of $G$.
Clearly, the maximum number of variables 
in a bag of the resulting tree decomposition is $w+1+\abs{U}$,
so the treewidth of $G$ is at most 
$w+\abs{U} = w + \abs{V(G)} - \abs{V(H)}$.
\end{proof}

\begin{figure}[tb]
  \centering
  \begin{tikzpicture}
    \filldraw[lightgray] (0,0) rectangle (1,1);
    \filldraw[lightgray] (1,1) rectangle (3,2);
    \filldraw[lightgray] (3,2) rectangle (4,3);
    \filldraw[black] (0,0) circle (2pt) node[anchor=east] {$c_1$};
    \filldraw[black] (0,1) circle (2pt) node[anchor=east] {$h_1$};
    \filldraw[black] (1,1) circle (2pt) node[anchor=south east] {$c_2$};
    \filldraw[black] (1,2) circle (2pt) node[anchor=east] {$h_2$};
    \filldraw[black] (3,2) circle (2pt) node[anchor=south east] {$c_3$};
    \filldraw[black] (3,3) circle (2pt) node[anchor=east] {$h_3$};
    \filldraw[black] (4,3) circle (2pt) node[anchor=west] {$c_4$};

  \end{tikzpicture}
\caption{Construction from the proof of 
Theorem~\ref{thm:cdc-lowerbound} for $k=4$.
}
\label{fig:colouring2cdc}
\end{figure}
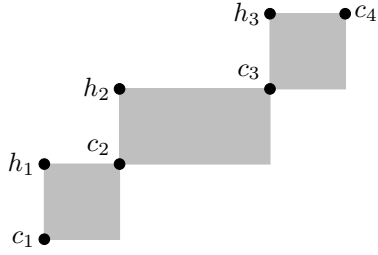

\noindent 
Now we are ready to establish our first lower bound.

\begin{theorem} \label{thm:cdc-lowerbound}
  CSP$(\CDCdisjlang)$ cannot be solved in 
  $2^{o(w\log{w})} \cdot n^{O(1)}$ time
  on instances with primal treewidth $w$ unless the ETH fails.
\end{theorem}

\begin{proof}

Let $(G,k)$ be an instance of the {\sc Vertex Colouring} problem.
We construct an instance $\inst$ of CSP$(\CDCdisjlang)$
that is satisfiable if and only if $G$ is $k$-colourable.

First, we introduce variables $z_v$ for all $v \in V(G)$,
variables $c_i$ for all colours $i \in \{1,\ldots,k\}$, 
and auxiliary variables $h_j$ for all $j \in \{1,\ldots,k-1\}$.
Then, we add the following constraints:
\begin{enumerate}[(C1)]
\renewcommand{\theenumi}{(C\arabic{enumi})}
\renewcommand\labelenumi{(C\arabic{enumi})}
  \item \label{ite:c1}
  $c_{i} \{ {\rm SW} \} c_{i+1}$, 
  $h_{i} \{ {\rm N} \} c_{i}$ and
  $h_{i} \{ {\rm W} \} c_{i+1}$ 
  for all $i \in \{1,\ldots,k-1\}$;
  \item \label{ite:c2}
  $z_v \{ {\rm =,NE} \} c_1$,
  $z_v \{ {\rm =,SW} \} c_k$
  for all $v \in V(G)$;
  \item \label{ite:c3}
  $z_v \{ {\rm SW,=,NE} \} c_i$ 
  for all $v \in V(G)$ and $i \in \{2,\ldots,k-1\}$;
  \item \label{ite:c4} 
  $z_v \{ \overline{\rm SE} \} h_j$
  for all $v \in V(G)$ and $j \in \{1,\ldots,k-1\}$, 
  where $\overline{{\rm SE}}$ is the negation of ${\rm SE}$, i.e. 
  $\overline{\rm SE} = \{ {\rm S, SW, W, NW, N, NE, E, =}\}$;
  \item \label{ite:c5}
  $z_u \{ {\rm SW, NE} \} z_v$ for all $(u,v) \in E(G)$.
\end{enumerate}

Towards proving correctness of the reduction,
let $\chi : V(G) \rightarrow \range{1}{k}$ 
be a proper $k$-colouring of $G$.
Define an assignment $f$ for $\inst$ by setting
$f(c_i) = (i,i)$ for all $i \in \{1,\ldots,k\}$,
$f(h_j) = (j,j+1)$ for all $j \in \{1,\ldots,k-1\}$ and
$f(z_v) = (\chi(v), \chi(v))$ for all $v \in V(G)$.
It is straightforward to verify that the 
constraints~\ref{ite:c1}--\ref{ite:c4} are satisfied by this assignment.
The constraints in~\ref{ite:c5} are also satisfied because
$\chi$ is a proper colouring of $G$.

In the opposite direction, let $f$ 
be a satisfying assignment for $\inst$.
Assume that $f(c_i) = (a_i, b_i)$ for all $i \in \{1,\ldots,k\}$.
Constraints~\ref{ite:c1} imply that 
$f(h_j) = (a_j, b_{j+1})$ for all $j \in \{1,\ldots,k-1\}$.
Figure~\ref{fig:colouring2cdc} shows an example of the construction.
Consider $z_v = (x_v, y_v)$ for an arbitrary $v \in V(G)$.
By~\ref{ite:c2} and~\ref{ite:c3}, $z_v$ can only take values inside the rectangles 
with corners $(a_i,b_i)$, $(a_i,b_{i+1})$, 
$(a_{i+1},b_i)$ and $(a_{i+1},b_{i+1})$
(shaded in the figure)
excluding the boundary except for 
the bottom left and top right corners.
Constraints~\ref{ite:c4} forbid $z_v$ from taking values inside the rectangles,
which leaves only the corners as possible values.
Thus, $f(z_v) \in \{ f(c_1), \ldots, f(c_k) \}$ for all $v \in V(G)$,
and we can define a colouring $\chi$ by setting
$\chi(v) = i$ whenever $f(z_v) = f(c_i)$. 
Note that Constraints~\ref{ite:c5} imply that $f(x_u) \neq f(x_v)$
whenever $(u,v) \in E(G)$. Therefore $\chi$ is a proper colouring.

Now consider the structure of the instance $\inst$. 
Denote the primal treewidth of $\inst$ by $w$ 
and the treewidth of $G$ by $w_G$.
Observe that the primal graph of $\inst$ consists of $G$ with
$2k - 1$ additional vertices for $c_1,\ldots,c_k$ 
and $h_1,\ldots,h_{k-1}$.
By Lemma~\ref{lem:treewidth-add-variables}, $w \leq w_G + 2k - 1$.
Furthermore, by Lemma~\ref{lem:colour-treewidth},
$G$ can be coloured with at most $w_G + 1$ colours.
Thus, we can safely assume that $k \leq w_G$ and, 
consequently, $w < 3 w_G$.
Therefore, if CSP$(\CDCdisjlang)$ admits
a $2^{o(w\log{w})} \cdot n^{O(1)}$ algorithm, 
then so does the Vertex Colouring and this contradicts the ETH
by Theorem~\ref{thm:colouring-lowerbound}.
\end{proof}

We continue by establishing lower bounds on 
CSP$(\IAdisjlang)$ and CSP$(\BAdisjlang)$.
To prove the result, we use the following lemma:

\begin{sloppypar}
\begin{lemma} \label{lem:cdc2allen}

  There is a polynomial-time reduction from 
  CSP$(\CDCdisjlang)$ to CSP$(\IAdisjlang)$
  that preserves the primal graph of the instance. 
\end{lemma}
\end{sloppypar}
\begin{proof}
Let $(V,C)$ be an instance of CSP$(\CDCdisjlang)$
and suppose $f$ is a satisfying assignment.
Denote $f(v)$ by $(v_1,v_2)$ for all $v \in V$.
We may assume without loss of generality 
that $v_1 < v_2$ for all $v \in V$, 
since all points in the image of $f$ can be translated
into the second quadrant of the coordinate plane,
where $v_1 < 0$ and $v_2 > 0$.
Thus, every pair of points $(v_1,v_2)$ 
can be viewed as an interval $[v_1,v_2]$.
With this in mind, we produce an instance 
$(V',C')$ of CSP$(\CDCdisjlang)$ with $V' = V$ 
by converting every relation in $C$ into a $\IAdisjlang$ relation
according to the rules in Table~\ref{tb:cdc-rules}.
The disjunction of any subset of the CDC relations is obtained 
by taking the disjunction of their converted counterparts.

\begin{table}[tb]
\begin{center}
  \begin{tabular}{ | c | c | c | }
  \hline
  {\rm CDC} & Definition & {\rm Allen} \\ 
  \hline
  \hline
  {\rm =}  & $x_1 = y_1$ and $x_2 = y_2$ & {\rm e} \\
  \hline
  {\rm N}  & $x_1 = y_1$ and $x_2 > y_2$ & {\rm si} \\
  \hline
  {\rm E}  & $x_1 > y_1$ and $x_2 = y_2$ & {\rm f} \\
  \hline
  {\rm S}  & $x_1 = y_1$ and $x_2 < y_2$ & {\rm s} \\
  \hline
  {\rm W}  & $x_1 < y_1$ and $x_2 = y_2$ & {\rm fi} \\
  \hline
  {\rm NE} & $x_1 > y_1$ and $x_2 > y_2$ & {\rm \{oi,mi,pi\}}  \\
  \hline
  {\rm SE} & $x_1 > y_1$ and $x_2 < y_2$ & {\rm d} \\
  \hline
  {\rm SW} & $x_1 < y_1$ and $x_2 < y_2$ & {\rm \{p,m,o\}}  \\
  \hline
  {\rm NW} & $x_1 < y_1$ and $x_2 > y_2$ & {\rm di} \\
  \hline
  \end{tabular}  
\end{center}
\caption{Transformation of CDC relations into Allen relations}
\label{tb:cdc-rules}
\end{table}

Equivalence of $(V,C)$ and $(V',C')$ follows from the definitions
of the basic relations of Cardinal Direction Calculus and 
Allen's Interval Algebra.
Clearly, the reduction requires polynomial time.
Furthermore, there is a constraint in~$C'$ over a pair of
variables if and only if there is a constraint in $C$
over the same pair of variables.
Thus, $(V,C)$ and $(V',C')$ have the same primal graph.
\end{proof}

\begin{corollary} \label{cor:allen,ba-lowerbound}
  CSP$(\IAdisjlang)$ and CSP$(\BAdisjlang)$
  cannot be solved in $2^{o(w\log{w})} \cdot n^{O(1)}$ time
  on instances with primal treewidth $w$ unless the ETH fails.
\end{corollary}

\begin{proof}
The result for CSP$(\IAdisjlang)$ follows directly by
combining Theorem~\ref{thm:cdc-lowerbound} with
Lemma~\ref{lem:cdc2allen}.
Note that $\BAlang$ generalises $\IAlang$
(namely, $\IAlang$ is $\BAlang$ for $d=1$).
Thus, the lower bound also holds for CSP$(\BAdisjlang)$.
\end{proof}

As for RCC5 and RCC8, the $2^{O(w^2)}$ term in the running time 
cannot be improved without introducing new ideas for the algorithm.
More precisely, we show that $\tau(w) = 2^{\Theta(w^2)}$ 
for CSP$(\B_{\rm RCC5})$ by the following observation:

\begin{sloppypar}
\begin{observation}
  There are $2^{\Theta(w^2)}$ complete satisfiable instances 
  of CSP$(\B_{\rm RCC5})$ with $w$ variables.
\end{observation}
\end{sloppypar}
\begin{proof}
First, note that there are $\abs{\B_{\rm RCC5}}^{\binom{w}{2}} = 2^{O(w^2)}$
not necessarily satisfiable instances of CSP$(\B_{\rm RCC5})$
with $w$ variables. 
Now, consider complete instances $(V,C)$ of this problem 
with $V = \{ v_1, \ldots, v_w \}$, where the constraints over each 
pair of variables are either {\sf DR} or {\sf PO}, chosen arbitrarily.
We claim that every such instance is satisfiable, and since there are 
$2^{\binom{w}{2}} = 2^{O(w^2)}$ of them, this yields the result.

Recall that in RCC5, the domain consists of all subsets 
of a topological space.
Note that the subsets need not be internally connected.
We refer to internally connected subsets as {\em regions}.
To prove the claim, we construct an assignment $f$ that assigns
a subset of disjoint regions to every variable.
For convenience, we consider two sets of regions:
$X_i$ for all $i \in \{1,\ldots,w\}$ and 
$Y_{i,j}$ for all $i,j \in \{1,\ldots,w\}$ with $i<j$.
First, we set $f(v_i) = \{ X_i \}$ for all $i \in \{1,\ldots,w\}$.
Then, for every pair $i,j \in \{1,\ldots,w\}$ with $i<j$ such that
{\sf PO}$(v_i, v_j)$ is in $C$,
we add $Y_{i,j}$ to both $f(v_i)$ and $f(v_j)$.

If {\sf DR}$(v_i, v_j)$ is in $C$,
then $f(v_i) \cap f(v_j) = \varnothing$,
so $f(v_i)$ and $f(v_j)$ are disjoint.
Otherwise, if {\sf PO}$(v_i, v_j)$ is in $C$,
then $f(v_i) \cap f(v_j) \neq \varnothing$,
$f(v_i) \setminus f(v_j) \neq \varnothing$ and
$f(v_j) \setminus f(v_i) \neq \varnothing$,
so $f(v_i)$ and $f(v_j)$ partially overlap.
Thus, $f$ is a satisfying assignment for $(V,C)$
and this completes the proof.
\end{proof}

RCC8 is a generalisation of RCC5, and the same result holds for RCC8 by the same arguments.

\section{Applications Based on Patchwork}
\label{sec:applicability}

We analyse the applicability of our fpt result (Theorem~\ref{thm:fpt}) in
this section. 
The patchwork property has not been directly verified for many formalisms---the list
in Corollary~\ref{cor:rcc,aia,ba,cdc} is quite meager.
However, it has been verified implicitly for wide classes of relations,
and this is something that can be exploited.
We first connect the patchwork property with the {\em amalgamation property} and
{\em homogeneous} structures. 
This allows us to use model-theoretical concepts and results to identify
interesting classes of relations that have the patchwork property. In the final
step, we demonstrate how these ideas can be used on concrete examples --- we study 
constraint languages that are first-order definable in $({\rationals};<)$
and phylogeny languages.

\subsection{Patchwork, Amalgamation and Homogeneity}

When analysing PP from a model-theoretic angle, it is convenient to view
CSPs in terms of homomorphisms.
A {\em homomorphism} for $\tau$-structures $\A,\B$ is a mapping 
$h: \A \rightarrow \B$ that preserves each relation of $\A$, i.e. if $(a_1, \ldots , a_k) \in R^\A$ for some $k$-ary
relation symbol $R \in \tau$, then $(h(a_1), \ldots , h(a_k)) \in R^\B$.
Let $\B$ be a structure with a (not necessarily finite) signature $\tau$. 
CSP$(\B)$ is then the following decision problem:

\medskip

\noindent
{\sc Instance.} A finite $\tau$-structure $\A$.\\
{\sc Question.} Is there a homomorphism from $\A$ to $\B$?

\medskip

It is well known that this definition coincides with
the definition given earlier; this is, for instance, discussed in~\cite[Section~2]{Bodirsky:Jonsson:jair2017}.
We will use an analogue of subinstances for $\tau$-structures:
a $\tau$-structure $\A$ is a {\em substructure} of a $\tau$-structure $\B$ if and only if
(1) the domain of $\A$ is a subset of the domain of  $\B$ and 
(2) for each $R \in \tau$, the tuple $\overline{a}$  is in $R^\A$ if and only if $\overline{a}$ is in $R^\B$.
We need several kinds of homomorphisms in what follows.
A {\em strong} homomorphism additionally satisfies the only if direction in the definition
of a homomorphism, i.e. it also preserves the complements of relations.
An {\em embedding} is an injective strong homomorphism.
An {\em isomorphism} is a surjective (and thus bijective) embedding, and
an {\em automorphism} is an isomorphism from~$\A$ to itself.
Let $\aut{\A}$ denote the set of automorphisms of $\A$.
It is easy to verify that $\aut{({\rationals};<,=,>)}$ contains the
function $f(x)=a+x$ for arbitrary $a \in {\rationals}$ and the function
$g(x)=b \cdot x$ for every rational number $b > 0$. However, the
set of automorphisms contains {\em many} other types of functions.

\begin{sloppypar}
We connect the definition of patchwork with the {\em amalgamation property}~(AP).
A class ${\cal K}$ of $\tau$-structures has AP
if for every $\B_1, \B_2 \in {\cal K}$ such that
their maximal common substructure $\A$
contains all elements that are both in $\B_1$ and $\B_2$,
there exists $\C \in {\cal K}$ (called an {\em amalgam})
and embeddings
$f_1 : \B_1 \rightarrow \C$ and $f_2 : \B_2 \rightarrow \C$ 
such that $f_1(a) = f_2(a)$ for every $a \in \A$.
Let $\D$ be a countable $\tau$-structure.
$\age{\D}$ denotes the class of all finite $\tau$-structures that embed into $\D$.
Various connections between patchwork and amalgamation concepts have been hinted upon in the literature many times (see e.g. Bodirsky and Jonsson~\cite{Bodirsky:Jonsson:jair2017},
Huang~\cite{Huang:kr2012}, and 
Li et al.~\cite{Li:etal:ijcai2009,Li:etal:ecai2008}) 
but the details have not been clearly spelled out. 
Baader \& Rydval~\cite{Baader:Rydval:ijcar2020} proved the following result.
\end{sloppypar}

\begin{theorem} \label{thm:BR-patchwork}
Let $\D$ be a JE$^+$PDJD structure. If $\age{\D}$ has the amalgamation property, then $\D$ has the patchwork property.
\end{theorem}

Their results do not apply directly to structures that are $k$-ary.
We complement Theorem~\ref{thm:BR-patchwork} by showing that the same implication holds for
$k$-ary JEPD structures that contain the $k$-ary equality relation.

\begin{theorem} \label{thm:ap-patchwork}
Let $\D$ be a $k$-ary JEPD $\tau$-structure with domain $D$
and assume that the $k$-ary equality relation $Eq_k = \{ (d,\ldots,d) \in D^k \mid d \in D \}$ is in $\D$. 
If $\age{\D}$ has the amalgamation property, then $\D$ has the patchwork property.
\end{theorem}
\begin{proof}
Consider the instances 
$I_1 = (V_1, C_1)$, $I_2 = (V_2, C_2)$
of CSP($\D$) in Definition~\ref{def:patchwork}
as $\tau$-structures $\I_1$, $\I_2$.
Note that the intersection $I_1[V_1 \cap V_2] = I_2[V_1 \cap V_2]$ 
viewed as a $\tau$-structure $\A$ is the maximal common substructure of $\I_1$, $\I_2$
and contains all elements that appear in both of them.
To apply AP, we need to show that $\I_1$ and $\I_2$ embed into $\D$.
Recall that an embedding is an injective strong homomorphism. 

The remainder of the proof applies for all $i \in \{1,2\}$.
Since $I_i$ is satisfiable, there is a homomorphism $h_i : \I_i \rightarrow \D$.
Additionally, $I_i$ is complete and $\D$ has JEPD relations, 
so for all $R \in \tau$,
$(h_i(x_1),\ldots,h_i(x_k)) \in R^\D$ implies that the constraint 
$R(x_1,\ldots,x_k)$ is in $C_i$ and it is satisfied.
Hence, $h_i$ is a strong homomorphism.
To show that it is injective, we observe that
for all $x,y \in \I_i$, 
if $Eq_k(x,y,\ldots,y) \in C_{i}$, then $x=y$.
Otherwise, by completeness,
there is another $R \in \tau$ such that $R(x,y,\ldots,y) \in C_i$.
By PD, $R \cap Eq_k = \varnothing$, so $x \neq y$.
Thus, $h_i$ is injective, and ergo, an embedding.

We know that $\I_1, \I_2 \in$ $\age{\D}$ so, by AP, 
the amalgam of $\I_1$ and $\I_2$ is also in $\age{\D}$.
Note that the structure $\C$ defined by 
$(V_1 \cup V_2, C_1 \cup C_2)$
embeds into the amalgam.
Hence, it is homomorphic to $\D$ and 
$(V_1 \cup V_2, C_1 \cup C_2)$ is satisfiable.
\end{proof}

\medskip

Theorems~\ref{thm:BR-patchwork} and~\ref{thm:ap-patchwork} allow us to relate PP to some 
properties and results that have been successfully used in the study of CSPs.
To this end, we will use {\em homogeneity}. 
A  homogeneous structure $\A$ is a countable structure such that 
for every isomorphism $f: \B \rightarrow \C$ between finite substructures $\B, \C$ 
of~$\A$, there is an automorphism $f'$ of $\A$ extending $f$. 
Intuitively speaking, a homogeneous structure enjoys the following property: 
the surroundings of two isomorphic substructures always look {\em very} similar.
Homogeneity thus implies that the structure has an extremely high degree
of symmetry.
The following result is part of the classical {\em Fraïssé's Theorem}~\cite{Fraisse:cras53}.

\begin{theorem} \label{thm:homogeneous-ap}
$\age{\A}$ has AP when $\A$ is a countable homogeneous structure with
a countable signature.
\end{theorem}

Fraïssé's Theorem is explained
in most textbooks on model theory such as
Hodges~\cite{Hodges:1997:SMT:262326}.
Combining Theorems~\ref{thm:BR-patchwork}, \ref{thm:ap-patchwork}, and \ref{thm:homogeneous-ap}
gives us the following result.

\begin{corollary} \label{cor:homogeneous-pp}
Let $\D$ denote a countable homogeneous structure with
a countable signature.

\begin{enumerate}
\item\label{ite:cor:homogeneous-pp-1}
If $\D$ is JE$^+$PDJD, then $\D$ has PP, and

\item\label{ite:cor:homogeneous-pp-2}
if $\D$ is a $k$-ary JEPD structure that
contains the $k$-ary equality relation, then $\D$ has PP. 
\end{enumerate}
\end{corollary}

A large number of
homogeneous structures are known from the literature (see, for example, the surveys by Macpherson~\cite{Macpherson:dm2011} 
and Hirsch~\cite{Hirsch:jlc97}) and they play an important role in CSP research. In fact, after the Feder-Vardi conjecture on finite-domain CSPs
was settled (independently) by Bulatov~\cite{bulatov2017} and Zhuk~\cite{zhuk2020}, much of the 
complexity-oriented work has concentrated on
homogeneous infinite-domain CSPs. 
We note that all examples in Corollary~\ref{cor:rcc,aia,ba,cdc} can be formulated by homogeneous structures; for instance,
Hirsch~\cite{Hirsch:ai96} proved this for Allen's Interval Algebra and 
Bodirsky and Wölfl~\cite{Bodirsky:Wolfl:ijcai2011} for RCC8. 
A fact to keep in mind is that
one may have two structures $\A$ and $\B$ such that CSP$(\A)$ is the same computational problem as CSP$(\B)$,
$\A$ is homogeneous, but $\B$ is not homogeneous. 
This phenomenon is, for instance, discussed (in the context
of RCC8) by Bodirsky and Wölfl~\cite{Bodirsky:Wolfl:ijcai2011}
and Huang et al.~\cite{Huang:etal:ai2013} (in the context of temporal constraints).
A straightforward example is provided by the structures $({\rationals}; <)$ and
$({\naturals}; <_{\naturals})$ where $<_{\naturals}$ denotes the ordering on the natural numbers.
The structure $({\rationals}; <)$ is homogeneous 
while $({\naturals}; <_{\naturals})$ is not\footnote{
Consider $f(1) = 0$, which is a trivial isomorphism between the substructures $(\{1\}; <_{\naturals})$ and $(\{0\}; <_{\naturals})$,
but cannot be extended to an automorphism -- there is no way to choose $f(0)$ such that $f(0) < f(1)$.
}, and
CSP$(({\rationals}; <))$ and
CSP$(({\naturals}; <_{\naturals}))$ are the same computational problems.

\subsection{Examples}
\label{sec:examples}

The machinery presented above allows us to show fpt results for large families of CSPs.
Our first example is the set of CSPs ${\cal T}$ whose constraint languages consist of
finite subsets of $\boole{({\rationals}; <)}$.
Well-known CSPs in ${\cal T}$ are the Point Algebra~\cite{Vilain:Kautz:aaai86}, the ORD-Horn class~\cite{Nebel:Burckert:jacm95} and certain
scheduling problems~\cite{Mohring:etal:sicomp2004}, together with basic problems in complexity theory 
such as {\sc Betweenness} and {\sc Cyclic Ordering}~\cite{gj79}. 
Clearly, ${\cal T}$ contains many different CSPs based on non-binary relations and, in fact,
the CSPs with binary relations are
a subset of the Point Algebra and thus polynomial-time
solvable~\cite{Vilain:Kautz:aaai86}.
One ought to observe that the CSP for Allen's Interval Algebra is {\em not}
in ${\cal T}$ since its domain consists of the closed convex subsets of ${\rationals}$ and not of ${\rationals}$ itself, but
there is a straightforward reduction from Allen's Interval Algebra to a certain problem in ${\cal T}$.
The CSPs in ${\cal T}$ have been intensively studied in the literature: 
for instance, Bodirsky and Kára~\cite{Bodirsky:Kara:jacm2010}
proved that any CSP in ${\cal T}$ is either polynomial-time solvable or \NP-complete.

Arbitrarily choose CSP$(\Gamma)$ in ${\cal T}$. 
It is folklore that the structure ${\bf Q}=({\rationals};<\nobreak,>,=)$ is homogeneous
(see, for instance, Example 2.1.2 in Macpherson~\cite{Macpherson:dm2011} for a proof sketch).
The structure ${\bf Q}$ is obviously JEPD and it contains the binary equality relation,
so it has PP by Corollary~\ref{cor:homogeneous-pp}. 
Since CSP$({\bf Q})$ is decidable, it follows from Theorem~\ref{thm:fpt} that CSP$(\Gamma)$ is fpt.
This proves the following.

\begin{proposition}
Every problem in ${\cal T}$ is fpt parameterized by the treewidth of the primal graph.
\end{proposition}

Hirsch~\cite{Hirsch:jlc97} points out and discusses interesting homogeneous structures whose CSP can be
solved with the same approach as for ${\cal T}$.
Moreover, Hirsch~\cite{Hirsch:ai96} proposed 
studying the computational complexity of CSPs for {\em relation algebras}, with obvious applications in AI.
Inspired by this research programme,
Bodirsky and Knäuer~\cite{Bodirsky:Knauer:nsat} 
recently identified sufficient conditions
for homogeneity of relation algebras. 
Their results provide further examples of CSPs that are
covered by Theorem~\ref{thm:fpt}.

We continue with a more elaborate example that demonstrates that Theorem~\ref{thm:fpt} is also
useful for non-CSP problems.
{\em Phylogeny problems} are used for phylogenetic reconstruction in 
bioinformatics, but also in 
areas such as database theory, computational genealogy, and
computational linguistics. A recent overview can be found in 
Warnow~\cite{Warnow:CP}. The problem is intuitively the following: given 
a partial description of a tree, is there a tree that is
compatible with the given information? 
Many problems of this kind
are \NP-hard: concrete examples include the {\em subtree avoidance problem}~\cite{Ng:etal:dam2000},
the {\em forbidden triple problem}~\cite{Bryant:PhD}, and the {\em quartet consistency problem}~\cite{Steel:jc92}.
Fpt algorithms are thus an interesting option for solving phylogeny problems.
Our basic idea is to rephrase phylogeny problems as CSPs and then
apply Theorem~\ref{thm:fpt}.
We formalise this below, mostly following
Bodirsky et al.~\cite{Bodirsky:etal:lmcs2017}. 

Let $T$ be a {\em tree}, i.e. an undirected,
acyclic, connected graph, and let $r$ be the {\em root} of $T$. 
We only consider binary trees, i.e. all vertices except for the root have
either degree $3$ or $1$, and the root has either degree $2$ or $0$. The vertex set of $T$ is
denoted by $V(T)$ and the set of leaves $L(T) \subseteq V(T)$ consists of the vertices of degree $1$.
For arbitrary $u, v \in V(T)$, we say that $u$ {\em lies below} $v$ if the path from $u$ to the root $r$ passes through $v$. 
We say
that $u$ {\em lies strictly below} $v$ if~$u$ lies below $v$ and $u \neq v$. The {\em youngest common ancestor} ($yca$) of 
$S \subseteq V(T)$ is the vertex $u$ that lies above all vertices in $S$ and has maximal distance
from~$r$; this vertex is uniquely determined by $S$.
The {\em leaf structure} of $T$ is the $\{C\}$-structure $(L(T); C)$
where $(x, y, z) \in C$ if and only if $yca(\{y, z\})$ lies strictly below $yca(\{x, y, z\})$. 
Following the literature on phylogeny problems, we write 
$x|yz$ instead of $C(x, y, z)$.

An {\em atomic phylogeny formula} $\phi$ is a conjunction of formulas of the form $x|yz$
and 
$x = y$.
We say that $\phi$ with variables $V$ is {\em satisfiable} if there exists a rooted binary
tree $T$ and a mapping $s: V \rightarrow L(T)$ such that $\phi$ is satisfied by $T$ under~$s$.
The atomic phylogeny problem {\sc Aphyl}
is the computational problem with atomic phylogeny formulas as instances and
the question is whether the formula is satisfiable or not.
{\sc Aphyl} is connected to CSPs as follows.

\begin{theorem} \label{thm:phylogeny-csp}
 There exists a homogeneous $\{|,=\}$-structure $\A$
with a countable domain and the following property: an instance $I$ of {\sc Aphyl} is satisfiable if and
only if  $I$ (viewed as an instance of CSP$(\{|,=\})$)  homomorphically maps to $\A$.
 \end{theorem}
\begin{proof}
Use Proposition~2 in Bodirsky et al.~\cite{Bodirsky:etal:lmcs2017}.
\end{proof}

The relation $x|yz$ will be a basic relation in the CSP we are aiming for.
Since we need a JEPD set of relations as the basis for Theorem~\ref{thm:fpt}, the following
observation (see, for instance Bodirsky et al.~\cite[Section~2.1]{Bodirsky:etal:lmcs2017}) is useful.

\begin{observation}\label{obs-phylo}
Let $x,y,z$ be arbitrary leaves in an arbitrarily chosen rooted binary tree.
If $x|yz$, then it may be the case that $y = z$.
However, $x|yz$ implies that $x \neq y$ and $x \neq z$. Hence, we either have $x|yz$, $y|xz$, $z|xy$ or $x = y = z$.
\end{observation}

Assume that the structure $\A$ in Theorem~\ref{thm:phylogeny-csp} has domain $A$ and contains the relations $|'$ and $='$.
Let ${\bf P}$ denote the structure $(A;R_1,R_2,R_3,R_4)$ where
$R_1(x,y,z) \Leftrightarrow x|'yz$, $R_2(x,y,z) \Leftrightarrow y|'xz$, $R_3(x,y,z) \Leftrightarrow z|'xy$, and $R_4(x,y,z) \Leftrightarrow (x='y='z)$.

\begin{proposition}\label{prop:phylogeny-fpt}
Let $\Gamma$ be a finite subset of $\boole{{\bf P}}$.
Then CSP$(\Gamma)$ is fpt parameterized by the treewidth of the primal graph.
\end{proposition}
\begin{proof}
We know that ${\bf P} \subseteq \boole{(A;|',=')}$
and it is straightforward
to verify that~${\bf P}$ is homogeneous since $(A;|',=')$ is homogeneous---all relations
in ${\bf P}$ can be obtained by permuting the arguments of relations in $(A;|',=')$.
The structure~${\bf P}$ is JEPD by Observation~\ref{obs-phylo} and it contains the ternary equality relation.
Thus, ${\bf P}$ has PP by Corollary~\ref{cor:homogeneous-pp}.\ref{ite:cor:homogeneous-pp-2}. 
Finally, CSP$({\bf P})$ is solvable in polynomial time~\cite{Aho:etal:sicomp81} and
the proposition follows from Theorem~\ref{thm:fpt}.
\end{proof}

This proves that the three
examples of \NP-hard phylogeny problems that were discussed earlier are fpt. 
We exemplify this with the aid of the forbidden triple problem (the exact
formulations of the other two problems as CSPs based on relations in $\boole{{\bf P}}$ can be found in
Bodirsky et al.~\cite[Section~2.2]{Bodirsky:etal:lmcs2017}). This problem is the phylogeny problem concerning
formulas \[F(x,y,z) = \neg(x|yz)\]
and the parameterization is the primal treewidth of a conjunction of such formulas. 
The essence of Theorem~\ref{thm:phylogeny-csp} and Proposition~\ref{prop:phylogeny-fpt}
is that there exists a relation $R(x,y,z) \in \boole{{\bf P}}$ that exactly captures the formula $F$:
\[R(x,y,z) \equiv \neg R_1(x,y,z)\]
or, equivalently,
\[R(x,y,z) \equiv R_2(x,y,z) \vee R_3(x,y,z) \vee R_4(x,y,z).\]
Thus, the forbidden triple problem can be viewed as CSP$(\{R\})$. The transformation from the forbidden triple 
problem to CSP$(\{R\})$ is simply
to replace each formula $F(x,y,z)$ with the constraint $R(x,y,z)$. This operation
obviously preserves the treewidth of instances.

We remark that the disequality relation neq $=\{(a,b) \in A^2 \; | \; a \neq b\}$ is used to 
define relations in some phylogeny examples---note that ${\rm neq}(x,y) \Leftrightarrow \neg R_4(x,x,y) \Leftrightarrow 
R_1(x,x,y) \vee R_2(x,x,y) \vee R_3(x,x,y)$
so the relation ${\rm neq}$ is a member of $\boole{{\bf P}}$.

\section{Beyond Patchwork}
\label{sec:beyondPP}

There are interesting examples
of structures $\A$ that do not have PP.
An eminent example
is the {\em Branching Time Algebra} (BTA) \cite{Anger:etal:spie91}
which has been used, for example, in planning~\cite{Dean:Boddy:aij88}, as the
basis for temporal logics~\cite{Emerson:Halpern:jacm86}, and as the basis for
a generalisation of Allen's Interval Algebra~\cite{Ragni:wolfl:sc2004}. We note that, in particular, 
the complexity
of the branching variant of Allen's Interval Algebra has recently gained
attraction~\cite{Bertagnon:etal:time2020,Gavanelli:etal:time2018}.
In BTA, the past of a time point is linearly ordered, but the
future is only partially ordered (see Bodirsky~\cite[Section~5.2]{Bodirsky:InfDom} for
formal details). 
This implies that time becomes a directed tree-like structure
with four basic relations $=$, $<$, $>$ and $||$, 
meaning ``equal'', ``before'', ``after'' and ``unrelated'', respectively.

One may formulate BTA as a CSP$(\B_{\rm BTA})$ where 
$\B_{\rm BTA}$ is JEPD, 
but one cannot formulate the problem so that $\B_{\rm BTA}$ has PP; this follows from
adapting an argument by Hirsch~\cite[Section~4.1]{Hirsch:jlc97}.
Both CSP$(\B_{\rm BTA})$ and CSP$({\B}_{\rm BTA}^{\vee =})$ are solvable in polynomial time~\cite[Section~4.2]{Hirsch:jlc97}, but 
there are finite $\Gamma \subseteq \boole{\B_{\rm BTA}}$ such that
CSP$(\Gamma)$ is NP-hard~\cite{Broxvall:Jonsson:ai2003}.
It is natural to ask whether CSP$(\Gamma)$ is fpt when $\Gamma$ contains higher-arity relations defined over $\B_{\rm BTA}$.
We show this by exploiting {\em homogenisability}: a homogenisable structure is a structure that can be
extended with a finite number of new relations in order to make the expanded structure homogeneous~\cite{Covington:ijm90,Hartman:etal:ms2015}. Homogenisation has recently become an interesting tool
for analysing CSPs: examples include 
connections between homogenisation and local
consistency algorithms for CSPs~\cite{Atserias:Torunczyk:csl2016} and applications 
concerning logically defined CSPs~\cite[Section~4.3.3]{Bodirsky:InfDom}. 
More background information about homogenisable structures can be found in the survey by Macpherson~\cite{Macpherson:dm2011}
and the thesis by Ahlman~\cite{Ahlman:PhD}.
We will prove a general fpt result for homogenisable structures in Theorem~\ref{thm:homogenisable}, and thus
prove that Theorem~\ref{thm:fpt} can indeed be generalised to certain structures that do not have PP.
In particular, this result proves that CSP$(\Gamma)$ is fpt whenever $\Gamma$ is a finite subset of $\boole{\B_{\rm BTA}}$.
We divide the rest of this section into two parts, where the first part is concerned with $\omega$-categorical structures
and the second part describes how homogenisable structures can be used to obtain fpt results.

\subsection{$\omega$-categoricity}

We first remind the reader of the definition of $\omega$-categoricity:
the {\em (first-order) theory} of a $\tau$-structure $\A$ (denoted by ${\rm Th}(\A)$) is
the set of all first-order $\tau$-sentences (i.e. formulas without free variables)
that are satisfied by $\A$, and ${\bf A}$ is said to be $\omega$-categorical if ${\rm Th}({\bf A})$ 
has exactly one model up to isomorphism. 
The concept of $\omega$-categoricity plays a key role in the study of complexity
aspects of CSPs~\cite{Bodirsky:InfDom}, but it is also important from an AI perspective~\cite{Hirsch:ai96,Huang:kr2012,Jonsson:ai2018}.
Examples of such structures include all structures with a finite domain and all structures that
were presented in Section~\ref{sec:qstr}.
The relation between $\omega$-categorical structures and homogeneous structures can be summarised as follows.
We say that a theory $T$ admits {\em quantifier elimination} if for every formula $\phi(x_1, \ldots ,x_n)$,
there is a quantifier-free formula $\psi(x_1,\ldots,x_n)$ such that $T$
entails
\[\forall x_1 \cdots \forall x_n (\psi(x_1,\ldots,x_n) \Leftrightarrow \phi(x_1,\ldots,x_n)).\]

\begin{theorem} \label{thm:homovsomega}
Let ${\bf A}$ be a structure.

\begin{enumerate}
\item
If ${\bf A}$ is $\omega$-categorical, then ${\bf A}$ is
homogeneous if and only if ${\rm Th}({\bf A})$ admits quantifier elimination~\cite[Proposition~3.1.6]{Macpherson:dm2011}.

\begin{sloppypar}
\item
If ${\bf A}$ is a homogeneous 
structure with finite signature, then
${\bf A}$ is $\omega$-categorical~\cite[Corollary~3.1.3]{Macpherson:dm2011}.
\end{sloppypar}
\end{enumerate}

\end{theorem}

A useful property of $\omega$-categorical structures is that they can be refined into finite JE$^+$PDJD structures.
Let ${\bf A}=(A;R_1,\ldots,R_m)$ denote a relational structure (that is not necessarily $\omega$-categorical).
The {\em orbit} of a tuple $\bar{a}=(a_1, \ldots, a_k) \in A^k$ (denoted by $\Orb{\bar{a}}$) 
is the set
\[\{(g(a_1), \ldots , g(a_n)) \; | \; g \in \aut{{\bf A}}\}\]
The orbits of $k$-tuples in $A$ partition the set $A^k$: for arbitrary $\bar{a},\bar{b} \in A^k$, either 
$\Orb{\bar{a}}=\Orb{\bar{b}}$
or $\Orb{\bar{a}} \cap \Orb{\bar{b}} = \emptyset$, and
for every 
$\bar{a} \in A^k$ there exists a $\bar{c} \in A^k$ such that
$\bar{a} \in \Orb{\bar{c}}$.
If ${\bf A}$ is $\omega$-categorical, then the set $\{\Orb{\bar{a}} \; | \; \bar{a} \in A^k\}$ is finite
for every $k \in \naturals$; this is an important consequence of a result by 
Engeler, Ryll-Nardzewski and Svenonius (this theorem is covered
by most textbooks on model theory such as Hodges~\cite{Hodges:1997:SMT:262326}).
The following definition will simplify our presentation.

\begin{definition}
Let ${\bf A}=(D;R_1,\ldots,R_m)$ denote a relational structure and let $d \geq 1$. Define ${\bf A}^{\leq d}$ to be
the relational structure over $D$ whose relations are all orbits of at most $d$-ary tuples over $A$.
\end{definition}

We collect a few straightforward facts about ${\bf A}^{\leq d}$ and we note that
these are discussed in more detail by Baader \& Rydval~\cite[Section~4]{Baader:Rydval:ijcar2020}.
Facts~\ref{fact:1} and \ref{fact:2} are based on the observation that if a relation $R$ contains
the tuple $\bar{a} = (a_1,\dots,a_k)$, then $\Orb{\bar{a}} \subseteq R$, too (by the definition of automorphisms),
while Fact~\ref{fact:3} is a direct consequence of the observation on orbit size of $\omega$-categorical structures that was made above.

\begin{fact}\label{fact:1}
The structure ${\bf A}^{\leq d}$ is JE$^+$PDJD, but it is not a $k$-ary structure in general and it is not necessarily finite. 
\end{fact}

\begin{fact}\label{fact:2}
Every $m$-ary relation $R \in {\bf A}$ with $m \leq d$ can be viewed as the union
of relations in ${\bf A}^{\leq d}$ or, equivalently, 
\[R(x_1,\ldots,x_m) \equiv \bigvee_{S \in {\bf S}} S(x_1,\ldots,x_m).\]
for some ${\bf S} \subseteq {\bf A}^{\leq d}$.
\end{fact}

\begin{fact}\label{fact:3}
The structure ${\bf A}^{\leq d}$ is finite if ${\bf A}$ is $\omega$-categorical.
\end{fact}

\medskip

The following result connects structures ${\bf A}^{\leq d}$ with $\omega$-categoricity, homogeneity, and the patchwork property.

\begin{theorem}[Immediate consequence of Theorem~5 in Baader \& Rydval~\cite{Baader:Rydval:ijcar2020}] \label{thm:BaaderRydvalpatchwork}
Let ${\bf A}$ denote an $\omega$-categorical homogeneous relational structure containing at most $d$-ary relations for some $d \geq 2$.
Then ${\bf A}^{\leq d}$ has the patchwork property.
\end{theorem}

\subsection{Homogenisation}

\begin{sloppypar}
We will now present a result (Theorem~\ref{thm:homogenisable}) concerning fixed-parameter tractability of CSPs based on structures
that do not have PP. To illustrate the result,
we will come back to the branching time problem at the end of this section.
The proof will use various ways of defining relations
with the aid of logical formulas.
Thus, in addition to full first-order logic, we also need
fragments
where only certain logical operators are allowed: 
the {\em existential fragment} consists of formulas built using negation, conjunction,
disjunction, and existential quantification only, while the
{\em existential positive fragment} additionally disallows negation.
We emphasise that it is required that the equality relation
is allowed in existential (positive) definitions, which is a difference compared to the definitions
underlying the operation $\boole{\cdot}$.
We begin with a decidability result.
\end{sloppypar}

\begin{lemma} \label{lm:ex-pos-def-decidable}
Let ${\bf A}$ be a relational structure and assume that the relations $R_1,\ldots,R_k$ are existential positive definable in ${\bf A}$.
If CSP$({\bf A})$ is decidable, then CSP$(\{R_1,\ldots,R_k\})$ is also decidable.
\end{lemma}
\begin{proof}
Suppose that, for $i \in \{1,\ldots,k\}$, $R_i$ has definition 
\[\phi_i(x_1,\ldots,x_m) \equiv \exists u_1 \cdots u_n \psi_i(x_1,\ldots,x_m,u_1,\ldots,u_n)\]
where $\psi_i$ is quantifier-free. 
We assume (without loss of generality) that each $\psi_i$ is in DNF. The conversion to DNF can be done without
introducing any negations. Define the $(m+n)$-ary relation $R'_i$ such that
$R'_i(x_1,\ldots,x_m,u_1,\ldots,u_n) \equiv \psi_i(x_1,\ldots,x_m,u_1,\ldots,u_n)$.

Let $I=(V,C)$ be an instance of  CSP$(\{R_1,\ldots,R_k\})$. We first construct an equivalent instance
$I'=(V',C')$ of CSP$(\{R'_1,\ldots,R'_k\})$. Start by setting $V'=V$.
Arbitrarily choose a constraint $R_i(x_1,\ldots,x_m)$ in $C$.
Expand $V'$ with~$n$ new variables $u_1,\ldots,u_n$ and add the relation
$R_i'(x_1,\ldots,x_n,u_1,\ldots,u_n)$ to $C'$.
Repeat this process for all constraints in $C$. It is obvious that $I$ is satisfiable if and only if $I'$ is satisfiable.

Recall that the formulas $\psi_1,\ldots,\psi_k$ are in DNF and that they contain no negations. 
If $I'$ is satisfiable by an assignment $f:V' \rightarrow D$, then we can construct a certificate that witnesses this.
For each constraint $R'_i(x_1,\ldots,x_m,u_1,\ldots,u_m)$, pick one term in $\psi_i$ that is satisfied
by $f$ and put it into the set $S$. It follows that $S$ is satisfiable if it is viewed as a CSP instance in the obvious way.
Now, $S$ only contains relations in ${\bf A} \cup \{=\}$ --- recall that the equality
relation can be used in an existential positive definition but its negation
$\neg(x=y)$ cannot.
For every constraint $x=y$ in $S$, identify the variable $x$ with the variable $y$ and
note that the resulting set $S'$ only contains relations in ${\bf A}$ and that it
is satisfiable if and only if $S$ is satisfiable.
This suggests the following algorithm: enumerate all possible certificates for $S'$ and check whether at least
one of them is satisfiable. Only a finite number of certificates exist, since ${\bf A}$ is finite
and the decidability of CSP$({\bf A})$ implies decidability of the satisfiability test.
We conclude that  CSP$(\{R'_1,\ldots,R'_k\})$ is decidable and so is  CSP$(\{R_1,\ldots,R_k\})$.
\end{proof}

We will now focus on structures that are {\em model-complete cores}.
The exact definition is not important for our purposes, but a certain
characterisation of $\omega$-categorical model-complete cores is very important.

\begin{lemma}[Theorem 4.5.1 in Bodirsky~\cite{Bodirsky:InfDom}] \label{lm:modcompcorechar}
A countable $\omega$-categorical structure ${\bf A}$ is {\em model-complete} if and only if
every first-order formula over ${\bf A}$ is equivalent to an existential positive formula 
over {\bf A}.
\end{lemma}

Lemma~\ref{lm:modcompcorechar} can be viewed as a restricted type of quantifier elimination.
Model-complete cores are very useful when studying CSPs. It is known that every countable and
$\omega$-categorical structure ${\bf A}$ is homomorphically equivalent to an $\omega$-categorical
model-complete core ${\bf A}'$---this implies that CSP$({\bf A})$ and CSP$({\bf A}')$ can be
viewed as the same problem~\cite[Section~1.1]{Bodirsky:InfDom}. The model-complete core ${\bf A}'$ can, in various ways, be
considered to be a more ``structured'' object than ${\bf A}$ and thus be easier to work with~\cite[Section~4.5]{Bodirsky:InfDom}.
Much of the work on the complexity of CSPs has consequently focused on model-complete cores.

We are finally ready to prove the main result of this section. We stress that
the definition of homogenisation of a structure ${\bf A}$ requires that only
a finite number of relations are added to ${\bf A}$. Otherwise, the resulting
structure contains an infinite number of relations and this would prevent us from
applying Theorem~\ref{thm:fpt}.

\begin{theorem} \label{thm:homogenisable}
Let ${\bf B}$ be a countably infinite $\omega$-categorical structure with finite signature, and assume that ${\bf B}$ is a model-complete core
and that CSP$({\bf B})$ is decidable.
If ${\bf B}$ is homogenisable by relations that are first-order definable in ${\bf B}$, then CSP$(\G)$ is fpt
parameterized by the treewidth of the primal graph
for arbitrary finite $\G \subseteq \boole{{\bf B}}$.
\end{theorem}
\begin{proof}
Let ${\bf C}$ denote the finite homogeneous expansion of ${\bf B}$ and let $d$ denote the maximal arity of relations in ${\bf C}$.
The structure ${\bf C}$ is $\omega$-categorical by Theorem~\ref{thm:homovsomega}.2.

\begin{claim} \label{cl:homogenisable:1}
  CSP$({\bf C}^{\leq d})$ is decidable.
\end{claim}
We say that a relation $R \subseteq A^k$ is {\em preserved} by a function $f:A \rightarrow A$
if for every $(a_1,\ldots,a_k) \in R$, $(f(a_1),\ldots,f(a_k))$ is also in $R$.
Let $\Orbi$ be an orbit of $k$-tuples of ${\bf C}$. It follows immediately from the
definition of orbits that $\Orbi$ is preserved by every function in $\aut{{\bf C}}$.
This implies that $\Orbi$ is first-order definable in ${\bf C}$~\cite[Proposition~4.2.9]{Bodirsky:InfDom}, since
${\bf C}$ is a countable $\omega$-categorical structure.
A direct consequence is that $\Orbi$ is also first-order definable in ${\bf B}$, since~${\bf C}$ is a first-order definable
expansion of ${\bf B}$.

Recall that ${\bf B}$ is a countable $\omega$-categorical model-complete core,  so
every first-order formula over ${\bf B}$ is logically equivalent to an existential
positive formula over ${\bf B}$ by Lemma~\ref{lm:modcompcorechar}.
We conclude that $\Orbi$ has an existential positive definition in ${\bf B}$.
With this in mind, it follows that the relations in ${\bf C}^{\leq d}$ are existential positive definable in ${\bf B}$, since ${\bf C}^{\leq d}$ only
contains relations that are orbits of tuples of ${\bf C}$.
Lemma~\ref{lm:ex-pos-def-decidable} thus implies that CSP$({\bf C}^{\leq d})$ is decidable, since ${\bf C}^{\leq d}$ only contains a finite
number of relations by Fact~\ref{fact:3}.

\begin{claim} \label{cl:homogenisable:2}
  ${\bf C}^{\leq d}$ has PP. 
\end{claim}
The structure ${\bf C}$ is
an $\omega$-categorical homogeneous structure that contains at most $d$-ary relations, 
so it has PP by Theorem~\ref{thm:BaaderRydvalpatchwork}.

\begin{claim} \label{cl:homogenisable:3}
  ${\bf C}^{\leq d}$ is JEPD. 
\end{claim}
The structure ${\bf C}^{\leq d}$ is JE$^+$PDJD by Fact~\ref{fact:1}.

\medskip

Theorem~\ref{thm:fpt} combined with Claims~\ref{cl:homogenisable:1}-\ref{cl:homogenisable:3} 
and the fact that ${\bf C}^{\leq d}$ is a finite structure implies that 
CSP$(\T)$ is fpt for arbitrary 
$\T \subseteq \boole{{\bf C}^{\leq d})}$.
The structure~${\bf B}$ is $\omega$-categorical, so Facts~\ref{fact:2} and~\ref{fact:3} imply that every relation in ${\bf B}$
can be viewed as a finite union of relations in ${\bf B}^{\leq d}$.
We know that ${\bf B} \subseteq {\bf C}$,
so ${\bf C}^{\leq d}$ consequently contains one relation for each orbit of $d$-tuples in ${\bf B}$, i.e.
${\bf B}^{\leq d} \subseteq {\bf C}^{\leq d}$.
This implies that every relation in $\boole{{\bf B}}$ has a logically equivalent relation in $\boole{{\bf C}^{\leq d}}$.
We may (without loss of generality since $\G$ is finite) assume that we have a pre-computed table that, for every $R \in \G$, contains the corresponding
relation~$R'$ in $\boole{{\bf C}^{\leq d}}$. Given an instance $(V,C)$ of CSP$(\G)$, we
can thus convert it in polynomial time into an equivalent instance of CSP$(\T')$ where $\T'$ is a finite
subset of $\boole{{\bf C}^{\leq d}}$.
We conclude that CSP$(\G)$ is fpt.
\end{proof}

Let us now return to the branching time example that we discussed in the beginning of this section.
Consider a structure $\B_{\rm BTA}=(B;=,<,>,||)$ 
such that the Branching Time Algebra problem is the same
computational problem as CSP$(\B_{\rm BTA})$. Two suitable structures have been
pointed out by Bodirsky~\cite[Section~5.2]{Bodirsky:InfDom}: they are referred to as ${\mathbb S}$ and ${\mathbb T}$. 
We do not consider ${\mathbb S}$ in what follows, since it is not a model-complete core. The structure ${\mathbb T}$, though, is
a countably infinite model-complete core that is $\omega$-categorical.
Hence, we let $\B_{\rm BTA}$ coincide with ${\mathbb T}$. 
Note that Bodirsky views these structures as having three relations $\leq$, $\neq$ and $||$,
where $||$ allows both that two elements are unrelated or that they are equal. This difference
is irrelevant in our setting; for instance, $x < y$ holds if and only if both $x \leq y$ and $x \neq y$ hold.
Now, 
$\B_{\rm BTA}$ is a countably infinite model-complete core that is $\omega$-categorical and,
additionally, Bodirsky et al.~\cite{Bodirsky:etal:jlc2018} have shown that
$\B_{\rm BTA}$ expanded by the relation
{
\setlength{\mathindent}{0.65cm} 
\[\{(x,y,z) \in B^3 \; | \; \exists u \in B \left( (u < x \vee u=x) \wedge (u > z \vee u||y) \wedge (z > u \vee z || u) \right) \}\]
}
\noindent
is homogeneous. Clearly, this relation is first-order definable in $\B_{\rm BTA}$.
We know that $\B_{\rm BTA}$ is polynomial-time solvable, so CSP$(\G)$ is fpt
for arbitrary finite $\G \subseteq \boole{\B_{\rm BTA}}$ by Theorem~\ref{thm:homogenisable}.

It may be illuminating to compare the Branching Time Algebra with its sibling -- the {\em partial-order time algebra} (PTA).
PTA has various applications in, for instance, the analysis of concurrent and distributed systems~\cite{Anger:toplas89,Lamport:jacm86}.
In PTA, both the past and the future of a time point are partially ordered. 
This implies that time becomes a partial order
with four basic relations $=_{\rm PTA}$, $<_{\rm PTA}$, $>_{\rm PTA}$ and $||_{\rm PTA}$, 
meaning ``equal'', ``before'', ``after'' and ``unrelated'', respectively.
The satisfiability problem for PTA can be formulated with
a countable finite homogeneous structure ${\bf D}$ (known as the {\em random partial order}) such that
${\bf D}$ is JEPD and CSP$({\bf D})$ is decidable. Thus, Theorem~\ref{thm:fpt} is directly
applicable in this case (since ${\bf D}$ has PP by Corollary~\ref{cor:homogeneous-pp}) and homogenisation is not necessary.
For more details concerning PTA, together with a complexity classification, see Kompatscher \& Van Pham~\cite{Kompatscher:VanPham:flap2018}.

\section{Discussion and Future Research}
\label{sec:discussion}

Huang et al.~\cite{Huang:etal:ai2013} proved that 
CSP$(\A)$ is in \XP whenever 
$\A$ is a binary constraint language with aNAP.
This property is PP restricted to binary relations, with
the completeness condition replaced by the algebraic closure condition.
aNAP is less restrictive than PP,
so it might be preferred in practical implementations
for some constraint languages.
However, in the worst case, using aNAP yields no advantage over using PP,
and it is only defined for binary languages.
We can thus conclude that our algorithm has a larger
scope of applicability than the algorithm by Huang et al. 

Bodirsky \& Dalmau~\cite{Bodirsky:Dalmau:jcss2013} show that CSP$(\A)$ is in \XP whenever
$\A$ is a countable structure that is $\omega$-categorical.
There are examples of $\omega$-categorical model-complete cores that cannot be made homogeneous by adding {\em any}
finite set of relations. A concrete example
based on the countable atomless Boolean algebra can be found in~\cite[Section~5.7]{Bodirsky:InfDom}.
We conclude that there are still $\omega$-categorical structures ${\bf A}$ for which we do not know
whether CSP$({\bf A})$ is fpt or not. Closing this gap is an obvious direction for future research.

There are many relevant CSP$(\G)$ where $\G$ is not $\omega$-categorical. 
Well-known examples include the unit interval algebra 
(i.e. Allen's Interval Algebra restricted to intervals of equal length~\cite{Peer:Shamir:tcs97})
and temporal problems that can express metric time, such as
the {\sc Simple Temporal Problem}, various disjunctive temporal problems
and extended variants of Allen's Interval Algebra~\cite{Dechter:etal:ai91,Krokhin:etal:sidma2004,Oddi:Cesta:ecai2000}.
These problems have been classified according to their parameterized complexity~\cite{Dabrowski:etal:aaai2021-dtp}.
Studying similar extensions of spatial formalisms
as well as other non-$\omega$-categorical CSPs is a natural future research direction.

Our algorithm solves CSPs over Cardinal Direction Calculus,
Allen's Interval Algebra and Block Algebra in $2^{O(w \log w)}$ time.
Under the Exponential Time Hypothesis, significantly improving
the dependence of $w$ is not possible in these cases.
However, for RCC5 and RCC8 the running time is slower,
since the number of certificates is $2^{O(w^2)}$.
Either proving a tight lower bound under plausible complexity assumptions
or finding a faster algorithm for RCC5 or RCC8 is an interesting future direction.
We remark, however, that an improved algorithm requires some new ideas 
(see also the discussion at the end of Section~\ref{sec:lowerbounds}).

Another plausible way forward is to consider parameterizations
that are less restrictive or orthogonal to primal treewidth. 
Examples that come to mind are the treewidth of the dual graph or the incidence graph and variants of hypertree width,
since these have been successfully used for efficiently solving CSPs as well as other combinatorial problems. 
However, these parameters can be ruled out for finite constraint languages,  since there
the treewidth of the primal, dual and incidence graphs, as well as 
all variants of hypertree width, are within a constant factor of each other~\cite{Samer:Szeider:jcss2010}.
This is not true for infinite constraint languages,
so here additional parameters are interesting to study.

\section*{Acknowledgements}
The first and third authors acknowledge support from the
Engineering and Physical Sciences Research Council (EPSRC project EP/V00252X/1).
The second and the fourth authors were supported by the Wallenberg AI, Autonomous Systems and Software Program (WASP) funded by the Knut and Alice Wallenberg Foundation.
In addition, the second author was partially supported by the Swedish Research Council (VR) under grant 2017-04112.

\bibliographystyle{plainurl}
\bibliography{references}

\end{document}